\documentclass{article}

\usepackage[utf8]{inputenc}
\usepackage[T1]{fontenc}
\usepackage[preprint]{neurips_2026}

\usepackage{graphicx}
\usepackage{amsmath,amssymb,amsthm}
\usepackage{mathtools}
\usepackage{booktabs}
\usepackage{multirow}
\DeclareUnicodeCharacter{0166}{\textT}
\newcommand{\textT}{T}
\usepackage{threeparttable}
\usepackage{algorithm}
\usepackage{algpseudocode}
\usepackage{hyperref}
\usepackage{enumitem}
\usepackage{adjustbox}   
\usepackage{xcolor}
\usepackage[most]{tcolorbox}

\tcbset{
  colback=white,
  colframe=black,
  arc=1.2mm,
  boxrule=0.5pt,
  left=1.4mm,
  right=1.4mm,
  top=1.1mm,
  bottom=1.1mm,
  fonttitle=\bfseries
}

\usepackage{tabularx}
\usepackage{makecell}

\definecolor{pfnpurple}{HTML}{7C4DD6}
\definecolor{epiteal}{HTML}{0F8C95}
\definecolor{aleaorange}{HTML}{E58F1F}

\newtheorem{proposition}{Proposition}
\newtheorem{corollary}{Corollary}
\theoremstyle{remark}

\tcbset{
  colback=white,
  colframe=black,
  arc=1.5mm,
  boxrule=0.6pt,
  left=1.2mm,right=1.2mm,top=1.2mm,bottom=1.2mm
}
\newtcolorbox{pfndefbox}[1][]{title={PFN Definition},#1}
\newtcolorbox{pfntrainbox}[1][]{title={Training Objective},#1}
\newtcolorbox{pfnremarkbox}[1][]{title={Remark},#1}
\newtcolorbox{pfnextbox}[1][]{title={Extension: Heavy-Tailed Noise},#1}

\newtcolorbox{mainideabox}[1][]{
  title={Key idea},
  colback=black!2,
  colframe=black!55,
  boxrule=0.5pt,
  arc=1.2mm,
  left=1.5mm,
  right=1.5mm,
  top=1.2mm,
  bottom=1.2mm,
  fonttitle=\bfseries,
  #1
}

\addtocontents{toc}{\protect\setcounter{tocdepth}{-1}}
\title{Decoupled PFNs: Identifiable Epistemic–Aleatoric Decomposition via Structured Synthetic Priors}

\author{
Richard Bergna$^{1}$ \quad
Stefan Depeweg$^{2}$ \quad
Jose Miguel Hern\'andez-Lobato$^{1}$\\
$^{1}$University of Cambridge \quad
$^{2}$Siemens AG
}

\date{March 2026}

\begin{document}

\maketitle

\begin{abstract}
Prior-Fitted Networks (PFNs) amortize Bayesian prediction by meta-learning over a synthetic task prior, but their standard output is a posterior predictive distribution over noisy observations. For sequential decision-making, such as active learning and Bayesian optimization, acquisition should prioritize epistemic uncertainty about the latent signal rather than irreducible aleatoric observation noise. We show that this epistemic--aleatoric split is not identifiable in general from the posterior predictive distribution alone, even when that distribution is known exactly. We then exploit a distinctive advantage of PFNs: because the synthetic data-generating process is under our control, each task can contain an explicit latent signal and noise function, and the generator can provide query-level labels for both the noiseless target and the observation-noise variance. We use these labels to train a decoupled PFN with separate latent-signal and aleatoric heads. The observation-level predictive is induced by convolving the latent signal distribution with the learned noise model. Empirically, epistemic-only acquisition mitigates the failure mode of total-variance exploration in noisy and heteroscedastic settings. In matched comparisons, decoupled models usually improve over tuned observation-level baselines, with the clearest gains in HPO; in broader sweeps, a decoupled model obtains the best average rank in both HPO and synthetic BO.
\end{abstract}

\section{Introduction}
\label{sec:introduction}

\begin{figure}[t]
    \centering
    \includegraphics[width=0.855\textwidth]{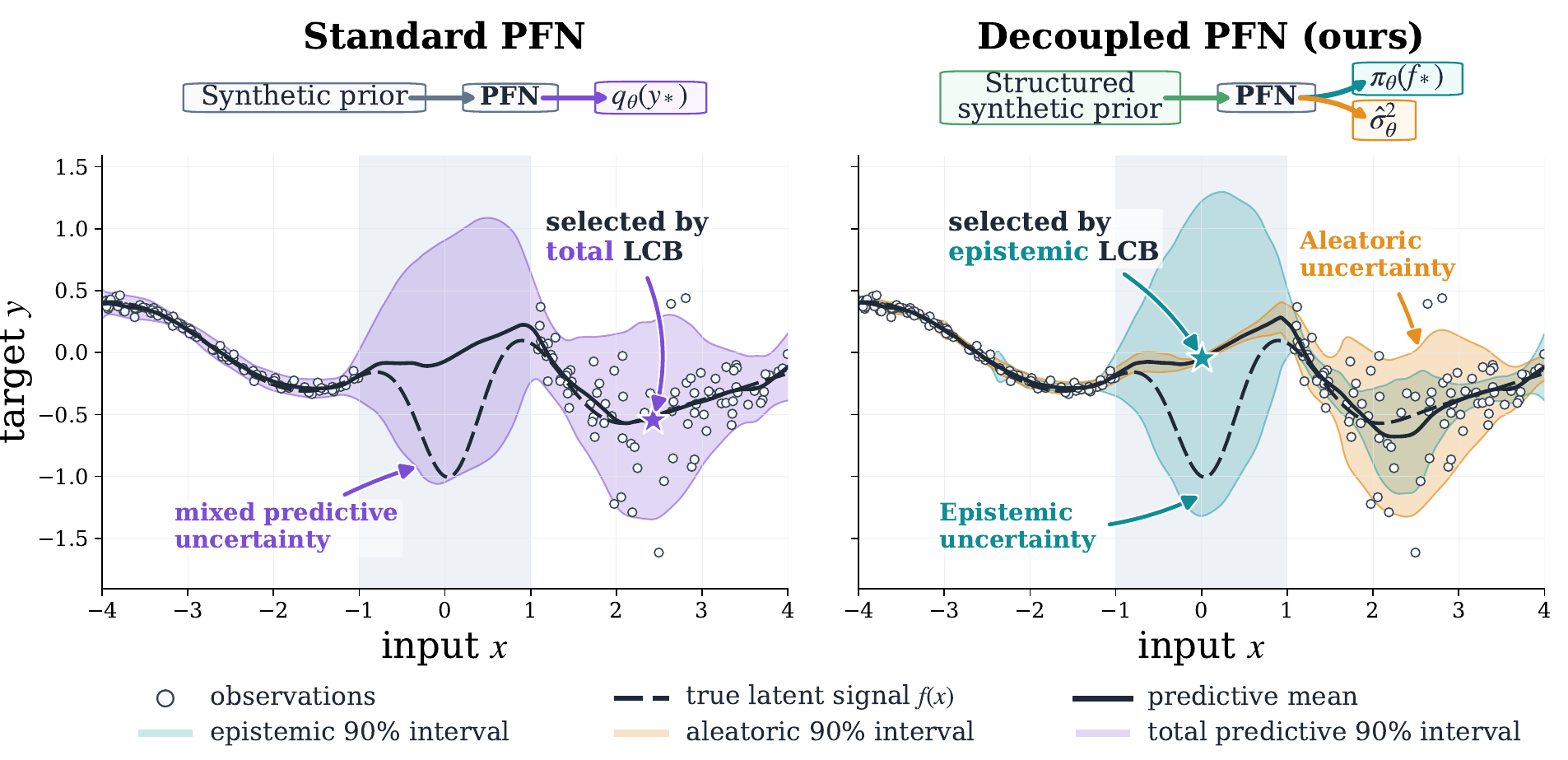}
\caption{
\textbf{Why decoupling matters for acquisition.}
In a heteroscedastic 1D task with an unsupported region (grey), total predictive uncertainty can draw a standard PFN toward high-noise observations. By separating latent-signal uncertainty from observation noise, our decoupled PFN enables epistemic-LCB to target the unsupported region instead. Stars mark selected queries.
}
    \label{fig:teaser_hetero_1d}
\end{figure}

In sequential decision-making, not all uncertainty is equally useful. Active learning, Bayesian optimization, and experimental design benefit from querying points that reduce uncertainty about the latent signal of interest \citep{mackay1992information,houlsby2011bayesian,jones1998efficient,srinivas2009gaussian}. By contrast, querying a point only because its observations are intrinsically noisy can waste budget: repeated measurements may reveal little about the underlying function. This motivates acquisition functions that distinguish \emph{epistemic} uncertainty, which is reducible by data, from \emph{aleatoric} uncertainty due to irreducible observation noise \citep{kendall2017uncertainties,depeweg2018decomposition}.

Prior-Fitted Networks (PFNs) are an attractive foundation for such problems. By meta-learning over synthetic task priors, they amortize Bayesian prediction: a single forward pass maps a context dataset and query input to a posterior predictive distribution \citep{muller2021transformers,nagler2023statistical}. Recent systems such as TabPFN, TabPFN v2, and TabICL demonstrate that this yields strong in-context predictors for tabular data without per-task optimization, ensembling, or extensive hyperparameter tuning \citep{hollmann2022tabpfn,hollmann2025accurate,qu2025tabicl}. However, the standard PFN objective trains the model to predict the noisy target~\(y_*\), so its output is an observation-level predictive \(p(y_* \mid x_*, \mathcal D)\) that combines uncertainty about the latent function with observation noise.

The consequence for acquisition is direct: a posterior predictive may be excellent for forecasting future observations while exposing the wrong uncertainty object for decision-making. Acquisition rules based on total predictive variance can confuse ignorance with noise and repeatedly select high-noise regions \citep{depeweg2018decomposition}. We show that this is not merely a modelling imperfection: the epistemic--aleatoric split is not identifiable from the marginal predictive distribution alone, even when that distribution is known exactly (Proposition~\ref{prop:non_identifiability}). A standard PFN trained only on noisy-target likelihood therefore has no objective-level reason to learn a unique decomposition.

Our key observation is that PFNs have access to a source of supervision
unavailable in ordinary supervised learning. Since pretraining tasks are
generated synthetically, we can augment the task prior so that each task
contains an explicit noiseless signal \(f_\tau\) and observation-noise function
\(\sigma_\tau^2\), with observations generated as
\[
    y = f_\tau(x) + \varepsilon,
    \qquad
    \varepsilon \sim \mathcal N(0,\,\sigma_\tau^2(x)).
\]
The generator provides query-level labels for both \(f_\tau(x_*)\) and
\(\sigma_\tau^2(x_*)\) as privileged auxiliary targets during pretraining; at
test time the model receives only \((\mathcal D, x_*)\), exactly as a standard
PFN. We use these labels to train a decoupled binned PFN whose latent-signal
head predicts a distribution over the noiseless target and whose aleatoric head
predicts the observation-noise variance. The observation-level predictive is
then induced by convolving the latent distribution with the learned noise model,
preserving the standard PFN interface while exposing a decision-relevant
epistemic signal for downstream acquisition.
Figure~\ref{fig:teaser_hetero_1d} illustrates the effect: the latent head
assigns uncertainty to unsupported regions while the aleatoric head tracks
input-dependent noise.

\begin{mainideabox}
A standard PFN is trained to predict noisy observations, so its posterior
predictive does not uniquely determine how uncertainty should be split into
latent-signal uncertainty and observation noise. We make this split supervised
during PFN pretraining: the synthetic generator defines both \(f_\tau(x)\) and
\(\sigma_\tau^2(x)\), and provides them as privileged training targets.
\end{mainideabox}

In summary, our contributions are:
\begin{itemize}[leftmargin=*, itemsep=0.25em, topsep=0.25em]
    \item We show that the posterior predictive learned by a standard PFN does
    not, in general, identify an epistemic--aleatoric decomposition, even when
    the predictive distribution is learned perfectly
    (Proposition~\ref{prop:non_identifiability}).

    \item We introduce a decoupled PFN trained with privileged latent-signal
    and noise-level labels from an augmented synthetic task prior, with
    separate latent-signal and aleatoric heads whose convolution induces the
    observation-level predictive.

    \item We use the resulting latent posterior to define epistemic-only
    acquisition functions for active learning and Bayesian optimization,
    showing that they mitigate the failure mode of total-variance exploration
    in noisy and heteroscedastic settings.
\end{itemize}

\section{Standard PFN Predictives Do Not Identify the Epistemic--Aleatoric Split}
\label{sec:pfn_bayesian}

We now formalize the Bayesian object learned by a standard PFN and show why it
does not determine a unique epistemic--aleatoric split.

Prior-Fitted Networks (PFNs) approximate Bayesian prediction under a synthetic
task prior. A task \(\tau \sim p(\tau)\) specifies a conditional distribution
\(p_\tau(y\mid x)\) over noisy observations. Given a context dataset $\mathcal D=\{(x_i,y_i)\}_{i=1}^{n}$,
and conditioning on the observed inputs, Bayesian inference under this prior
induces a posterior over tasks
\[
    p(\tau\mid \mathcal D)
    \propto
    p(\tau)\prod_{i=1}^{n}p_\tau(y_i\mid x_i).
\]
The corresponding posterior predictive distribution is
\begin{equation}
\label{eq:bayes_posterior_predictive}
    p(y_*\mid x_*,\mathcal D)
    =
    \int p_\tau(y_*\mid x_*)\,p(\tau\mid\mathcal D)\,d\tau .
\end{equation}
A PFN amortizes this computation by training a parametric predictor
\(q_\theta(\cdot\mid x_*,\mathcal D)\) on synthetic tasks:
\begin{equation}
\label{eq:standard_pfn_objective}
    \min_\theta\;
    \mathbb E_{\tau\sim p(\tau)}
    \mathbb E_{\mathcal D\sim p_\tau}
    \mathbb E_{(x_*,y_*)\sim p_\tau}
    \left[
        -\log q_\theta(y_*\mid x_*,\mathcal D)
    \right].
\end{equation}
Since log loss is a strictly proper scoring rule, the population optimum is the
Bayes posterior predictive in \eqref{eq:bayes_posterior_predictive}, assuming
sufficient model capacity and successful optimization. Thus, a standard PFN is
best understood as an amortized approximation to the Bayes posterior predictive
under the synthetic task prior.

For sequential decision-making, however, the posterior predictive is not always
the desired uncertainty object. It mixes uncertainty about the latent task with
irreducible observation noise. To see this, define
\[
    m_\tau(x):=\mathbb E[Y\mid x,\tau],
    \qquad
    s_\tau^2(x):=\mathrm{Var}(Y\mid x,\tau).
\]
Applying the law of total variance with respect to the posterior over tasks gives
\begin{equation}
\label{eq:total_variance}
    \mathrm{Var}(Y_*\mid x_*,\mathcal D)
    =
    \underbrace{
    \mathrm{Var}_{\tau\mid\mathcal D}
    \!\left(m_\tau(x_*)\right)
    }_{\text{epistemic}}
    +
    \underbrace{
    \mathbb E_{\tau\mid\mathcal D}
    \!\left[s_\tau^2(x_*)\right]
    }_{\text{aleatoric}} .
\end{equation}
The first term captures uncertainty over plausible latent tasks and is reducible
by additional data. The second captures residual variability that remains even if
the task were known. Equation~\eqref{eq:total_variance} shows that the full task
prior can define an epistemic--aleatoric decomposition. Standard PFN training,
however, constrains only the marginal predictive distribution over \(y_*\), not
these two components separately. This creates the ambiguity formalized below.

\begin{proposition}[Non-identifiability of an additive signal--noise split]
\label{prop:non_identifiability}
Fix a query--context pair \((x_*,\mathcal D)\), and suppose the marginal
predictive distribution is
\[
    p(y_*\mid x_*,\mathcal D)=\mathcal N(m,s^2),
    \qquad s^2>0.
\]
For every \(a\in(0,s^2)\), there exist independent random variables
\[
    F_* \mid x_*,\mathcal D \sim \mathcal N(m,a),
    \qquad
    \varepsilon_* \sim \mathcal N(0,s^2-a),
    \qquad
    Y_* = F_*+\varepsilon_*,
\]
such that
\[
    Y_* \mid x_*,\mathcal D \sim \mathcal N(m,s^2).
\]
Consequently, even in this Gaussian additive setting, the same marginal
predictive distribution is compatible with a continuum of additive
decompositions into latent-signal variance~\(a\) and residual-noise
variance~\(s^2-a\). Thus, the marginal predictive distribution alone does not
identify a unique epistemic--aleatoric variance split.
\end{proposition}

\begin{proof}
For any \(a\in(0,s^2)\), define \(F_*\) and \(\varepsilon_*\) as above. Since
the sum of independent Gaussian variables is Gaussian,
\[
    F_*+\varepsilon_*
    \sim
    \mathcal N(m+0,\;a+(s^2-a))
    =
    \mathcal N(m,s^2).
\]
However, the corresponding variance decomposition is
\[
    \mathrm{Var}(F_*\mid x_*,\mathcal D)=a,
    \qquad
    \mathrm{Var}(\varepsilon_*)=s^2-a.
\]
Different choices of \(a\) therefore induce different signal/noise variance
splits while yielding exactly the same marginal predictive distribution over
\(Y_*\). Hence the map from additive signal--noise decompositions to marginal
predictive laws is many-to-one, so the marginal predictive distribution alone
cannot identify the decomposition.
\end{proof}

Proposition~\ref{prop:non_identifiability} is pointwise and deliberately
minimal: it does not rule out identifying a decomposition from additional
assumptions, from structure tying the signal and noise processes across inputs,
or from direct supervision. Rather, it shows that a PFN trained only to match
\(p(y_*\mid x_*,\mathcal D)\) has no objective-level requirement to expose the
reducible and irreducible components separately. Our method supplies the missing
structure by defining the signal--noise split in the synthetic generator and
supervising both components during pretraining.

\section{Method}
\label{sec:method}

Motivated by Proposition~\ref{prop:non_identifiability}, we modify the PFN synthetic
task prior so that the signal--noise decomposition is specified explicitly during
pretraining. We sample augmented tasks
\[
    \tau = (p_\tau^x, f_\tau, \sigma_\tau^2),
\]
where \(p_\tau^x\) is an input distribution, \(f_\tau\) is a noiseless latent
signal, and \(\sigma_\tau^2\) is an observation-noise function. Observed targets
are generated as
\begin{equation}
\label{eq:signal_noise_model}
    y = f_\tau(x) + \varepsilon,
    \qquad
    \varepsilon \sim \mathcal N(0,\sigma_\tau^2(x)).
\end{equation}
Thus, the observation-level distribution \(p_\tau(y\mid x)\) is a derived
quantity, while the signal--noise decomposition is part of the generative model
itself.

The decomposition is therefore not recovered post hoc from the posterior
predictive distribution. Instead, the synthetic generator provides query-level
labels for the clean signal \(f_\tau(x_*)\) and the noise variance
\(\sigma_\tau^2(x_*)\), which are used as privileged auxiliary targets during
pretraining. The model itself receives only the noisy context
\(\mathcal D=\{(x_i,y_i)\}_{i=1}^n\) and the query input \(x_*\), both during
training and at test time. Before defining the binned architecture, we briefly specify the synthetic prior used to instantiate this supervised decomposition.

\paragraph{Synthetic heteroscedastic task prior.}
Tasks are sampled online during fine-tuning from an augmented prior that
separates the clean latent signal from observation noise. For each task, we first
sample a noiseless function \(f_\tau\) from a mixture of GP-style and
structural-causal priors, and then generate observations by adding Gaussian noise
according to Eq.~\eqref{eq:signal_noise_model}. The noise process is
heteroscedastic with high probability and homoscedastic otherwise. The resulting
clean query targets \(f_\tau(x_*)\) and noise variances
\(\sigma_\tau^2(x_*)\) are retained only as auxiliary training targets for the
decoupled models. Tuned and decoupled models are fine-tuned on the same
online-sampled task prior; exact sampler and optimization details are given in
Appendix~\ref{app:synthetic_pretraining}.

\subsection{Decoupled binned PFN}
\label{sec:bins_method}

We use fixed target bins
\[
    a_0 < a_1 < \cdots < a_K
\]
with centers
\[
    c_j = \frac{1}{2}(a_{j-1}+a_j), \qquad j=1,\dots,K.
\]
A standard PFN directly predicts observation-bin masses. Instead, we predict a
latent distribution over the noiseless signal and a separate aleatoric variance.
The epistemic head is formally a latent-signal head: it outputs logits
\(\ell^{(f)}\in\mathbb R^K\), defining
\begin{equation}
\label{eq:latent_categorical}
    \pi_{\theta,j}^{(f)}(x_*,\mathcal D)
    =
    q_\theta(f_*=c_j\mid x_*,\mathcal D)
    =
    \mathrm{softmax}(\ell^{(f)})_j .
\end{equation}
The uncertainty represented by this distribution is epistemic: it reflects
posterior uncertainty about the noiseless signal \(f_*\), rather than observation
noise.

The observation-level predictive distribution is obtained by passing the latent
signal distribution through the Gaussian noise model:
\begin{equation}
\label{eq:bins_predictive}
    p_k^{(y)}
    =
    \sum_{j=1}^K
    \pi_{\theta,j}^{(f)}(x_*,\mathcal D)
    T_{k\mid j}\!\left(\hat\sigma_\theta^2(x_*,\mathcal D)\right),
\end{equation}
where, writing $\hat\sigma_\theta
    =
    \sqrt{\hat\sigma_\theta^2(x_*,\mathcal D)}$,
\begin{equation}
\label{eq:bins_transition}
    T_{k\mid j}(\hat\sigma_\theta^2)
    =
    \Phi\!\left(\frac{a_k-c_j}{\hat\sigma_\theta}\right)
    -
    \Phi\!\left(\frac{a_{k-1}-c_j}{\hat\sigma_\theta}\right).
\end{equation}
Here \(T_{k\mid j}\) is the probability that a Gaussian centered at \(c_j\) with
variance \(\hat\sigma_\theta^2\) falls in observation bin \(k\), so the transition
matrix implements the binned analogue of convolving the latent signal distribution
with a Gaussian observation-noise model. In implementation, the transition
probabilities are computed over the finite bin range, clamped for numerical
stability, and renormalized to sum to one. This gives a truncated binned-Gaussian
approximation to the convolution while preserving the standard binned PFN
predictive interface and making the latent epistemic distribution explicit.

\subsection{Training objective}
\label{sec:bins_training}

For each synthetic query point, the generator provides the observed target \(y_*\),
the latent signal \(f_* = f_\tau(x_*)\), and the noise variance
\(\sigma_*^2 = \sigma_\tau^2(x_*)\). We train with three losses. The induced
observation distribution is trained with the usual binned negative log-likelihood:
\begin{equation}
\label{eq:Ly_bar}
    \mathcal L_y^{\mathrm{bar}}(\theta)
    =
    -\log p_{b(y_*)}^{(y)}
    +
    \log w_{b(y_*)},
\end{equation}
where \(b(y_*)\) is the bin containing \(y_*\) and \(w_k=a_k-a_{k-1}\). The latent
head is supervised using the bin containing the noiseless target:
\begin{equation}
\label{eq:Lf_cat}
    \mathcal L_f^{\mathrm{cat}}(\theta)
    =
    -\log
    \pi_{\theta,b(f_*)}^{(f)}(x_*,\mathcal D).
\end{equation}
The aleatoric head is supervised in log-variance space:
\begin{equation}
\label{eq:Lsigma_bar}
    \mathcal L_\sigma(\theta)
    =
    \left(
    \log \hat\sigma_\theta^2(x_*,\mathcal D)
    -
    \log \sigma_*^2
    \right)^2 .
\end{equation}
The full objective is

\begin{equation}
\label{eq:bins_full_objective}
    \mathcal L_{\mathrm{total}}(\theta)
    =
    \mathbb E_{\tau,\mathcal D,(x_*,y_*)}
    \biggl[
    \underbrace{\mathcal L_y^{\mathrm{bar}}}_{\textcolor{pfnpurple}{\text{predictive loss}}}
    +
    \lambda_f\,
    \underbrace{\mathcal L_f^{\mathrm{cat}}}_{\textcolor{epiteal}{\text{epistemic loss}}}
    +
    \lambda_\sigma\,
    \underbrace{\mathcal L_\sigma}_{\textcolor{aleaorange}{\text{aleatoric loss}}}
    \biggr].
\end{equation}

The three losses play distinct roles. The latent and aleatoric auxiliary losses
give the two heads their intended generator-defined semantics: the latent head is
trained to predict the clean signal, and the aleatoric head is trained to predict
the noise level. The observation loss is retained to ensure that the distribution
induced by convolving the latent head with the noise head remains a good predictive
distribution for the observed target \(y_*\), preserving the standard PFN
predictive interface.

For the auxiliary terms, the population minimizers have a simple interpretation.
The latent auxiliary loss recovers the conditional distribution of the discretized
clean signal \(B(f_\tau(x_*))\) given \((x_*,\mathcal D)\). Since
\(\mathcal L_\sigma\) is a squared loss in log-variance space, the aleatoric head
recovers the Bayes-optimal point predictor of log-variance, equivalently the
conditional geometric-mean summary
\[
\hat\sigma_\theta^2(x_*,\mathcal D)
=
\exp\left(
\mathbb E[\log \sigma_\tau^2(x_*)\mid x_*,\mathcal D]
\right).
\]
Thus, when there is posterior uncertainty over the noise function, the induced
observation-level predictive should be understood as a structured approximation
to the full Bayes predictive, rather than an exact marginalization over all
possible noise functions. Appendix~\ref{app:consistency} formalizes this
population-risk argument. At test time, epistemic uncertainty is computed from \(\pi_\theta^{(f)}\), for
example using
\[
    \mathrm{Var}_\theta(f_*)
    =
    \sum_{j=1}^K \pi_{\theta,j}^{(f)}c_j^2
    -
    \left(\sum_{j=1}^K \pi_{\theta,j}^{(f)}c_j\right)^2.
\]
Aleatoric uncertainty is given by the predicted noise variance
\(\hat\sigma_\theta^2(x_*,\mathcal D)\). Observation-level summaries, such as means and intervals, are computed from the induced predictive distribution \(p^{(y)}\).

\paragraph{Instantiation.}
We instantiate the method by adapting two public pretrained tabular foundation
models: TabICL~\citep{qu2025tabicl} and TabPFN v2.5~\citep{hollmann2025accurate}.
This gives four matched models. Decoupled-ICL (\textbf{Dec-ICL}) and
Decoupled-PFN (\textbf{Dec-PFN}) use the proposed latent-signal and aleatoric
heads, while Tuned-ICL and Tuned-PFN are observation-level tuned baselines with
the same pretrained backbones and 999-bin output discretization. The tuned
baselines are fine-tuned on the same synthetic task prior and optimization
protocol, but are trained only with the standard observation-level bar loss.
Thus, matched comparisons isolate the effect of the additional latent/noise
supervision and the resulting epistemic acquisition signal. Architecture,
partial-unfreezing, and loss-weight details are given in
Appendix~\ref{app:synthetic_pretraining}. 

\paragraph{Acquisition.} For BO and HPO, decoupled and tuned models use the same acquisition families, differing only in the uncertainty distribution used for acquisition. The main comparison uses log-Expected Improvement (LogEI): decoupled models compute it from the latent-signal moments \((\mu_f,v_{\mathrm{epi}})\), while tuned baselines compute it from the total observation-level predictive moments \((\mu_y,v_{\mathrm{tot}})\), since they do not expose a separate epistemic--aleatoric decomposition. Full acquisition details, including the closed-form Expected Improvement (EI) expression, lower confidence bound (LCB), and marginal Thompson-sampling (TS) ablations, are given in Appendix~\ref{app:acq_binned_pfn}.

\section{Related Work}
\label{sec:related_work}

\paragraph{Prior-fitted and tabular in-context predictors.}
We build on Prior-Fitted Networks (PFNs) and tabular in-context learning
\citep{muller2021transformers,nagler2023statistical,hollmann2022tabpfn,
hollmann2025accurate,qu2025tabicl}, and in particular PFN-based BO/HPO
surrogates such as PFNs4BO and freeze-thaw PFN methods
\citep{muller2023pfns4bo,rakotoarison2024context}. These works use PFNs as
posterior-predictive surrogates over observed targets. Our contribution is
orthogonal: we modify the synthetic pretraining objective and output
architecture so that the PFN exposes a latent-signal posterior separately from
observation noise.

\paragraph{Epistemic--aleatoric uncertainty and noisy acquisition.}
Separating epistemic from aleatoric uncertainty is standard in Bayesian deep
learning \citep{gal2016dropout,kendall2017uncertainties,
lakshminarayanan2017simple,depeweg2018decomposition} and central to acquisition
in active learning and Bayesian optimization
\citep{mackay1992information,houlsby2011bayesian,jones1998efficient,
srinivas2009gaussian}. Heteroscedastic GP and noisy BO methods model
input-dependent noise explicitly
\citep{kersting2007most,binois2018practical,cowen2022hebo}; we achieve an
analogous signal--noise separation inside a PFN via privileged supervision from
the synthetic generator.

\paragraph{BO/HPO baselines.} We compare against standard BO/HPO baselines, including GP-EI, BoTorch GP baselines, TPE, and random search \citep{jones1998efficient,balandat2020botorch,bergstra2011algorithms, bergstra2013making,akiba2019optuna}. These baselines provide practical points of comparison. The matched comparisons between Dec-ICL and Tuned-ICL, and between Dec-PFN and Tuned-PFN, serve a different purpose: they isolate whether acquisition improves when the surrogate exposes a latent epistemic posterior rather than only an observation-level predictive, while holding fixed the pretrained backbone, synthetic task prior, and acquisition family.

\section{Results}
We evaluate whether the decoupled latent-signal posterior provides a more useful
uncertainty signal for sequential decision-making. Our experiments cover three
settings: hyperparameter optimisation, synthetic Bayesian optimisation, and
active learning. The main text focuses on settings where the choice of
uncertainty signal directly affects the selected evaluations; observation-level
regression performance is reported in Appendix~\ref{app:regression}.

\paragraph{Main comparison and extended sweeps.}
The main text reports a compact comparison over eight core methods: Dec-ICL and
Dec-PFN, their matched observation-level tuned counterparts, Tuned-ICL and
Tuned-PFN, and four external baselines: a Gaussian-process surrogate with
expected improvement (GP-EI), a BoTorch GP baseline, TPE through Optuna, and
random search. For the PFN/ICL variants, the main sequential-optimisation
comparison uses log-Expected Improvement (LogEI). Decoupled models compute LogEI
from the latent-signal posterior, while tuned models compute LogEI from the total
observation-level predictive because they do not expose a separate epistemic
component. Appendix~\ref{app:acq_binned_pfn} defines the acquisition functions
for binned PFN outputs. The extended results in Appendices~\ref{app:hpo_full}
and~\ref{app:bo_full} report the full sweep over \(22\) method/acquisition
variants listed in Appendix~\ref{app:method_variants}, including TS, LCB,
EI/LogEI variants, total-variance ablations, multiple GP variants, TPE, and
random search.

Figure~\ref{fig:seqopt_summary} summarises the average-rank comparison. The
strongest gains from decoupling appear in hyperparameter optimisation, where
noisy and structured evaluations make the epistemic--aleatoric distinction
especially consequential. On synthetic benchmarks the decoupled methods remain
competitive, but the margins are smaller.

\begin{figure}[th]
    \centering
    \includegraphics[width=0.95\linewidth]{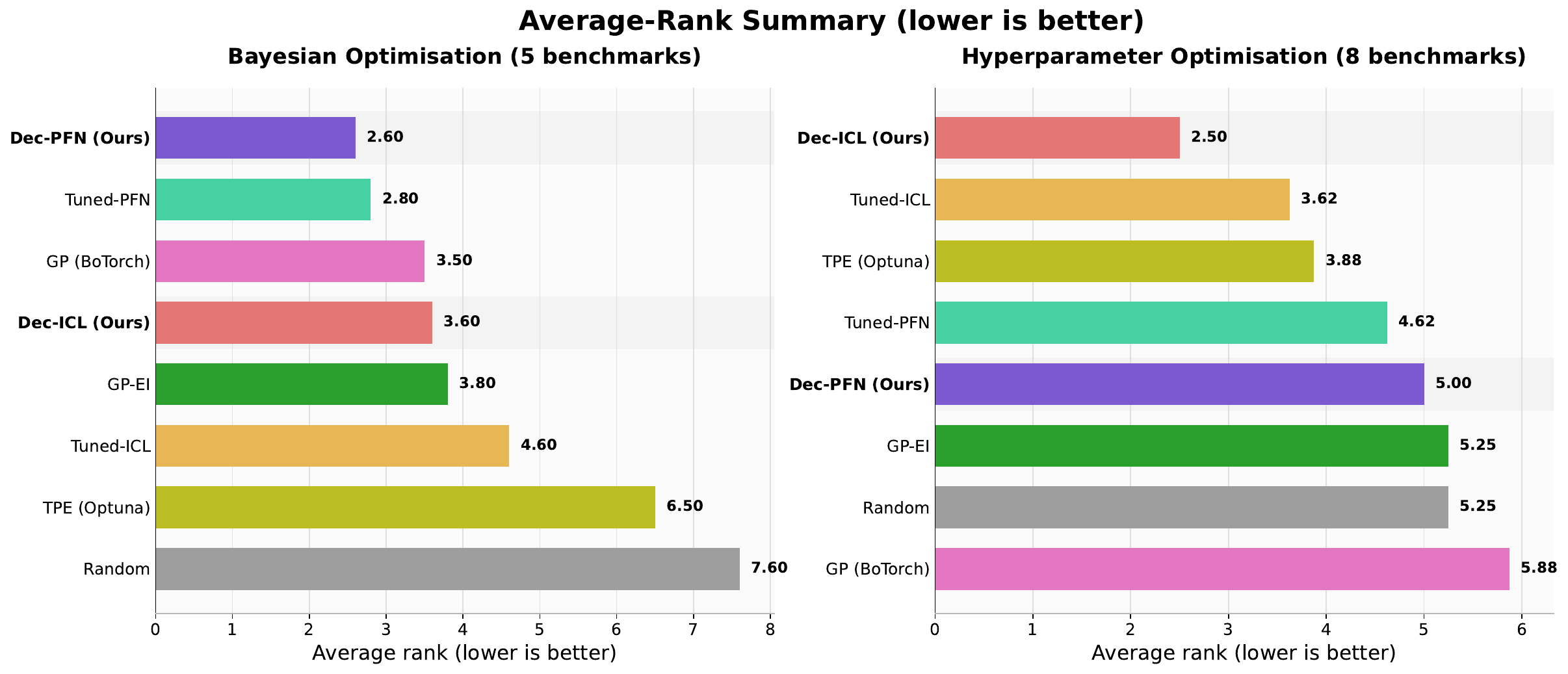}
\caption{
Average-rank summary across sequential optimisation benchmarks.
\emph{Left}: synthetic Bayesian optimisation (5 benchmarks).
\emph{Right}: hyperparameter optimisation (8 core benchmarks).
Ranks are computed among the eight core methods from the median final simple
regret over 10 seeds, with lower being better.
Decoupled methods use epistemic-only LogEI, while tuned counterparts use total-predictive LogEI.
Bold labels indicate our proposed methods.
}

    \label{fig:seqopt_summary}
\end{figure}

\subsection{Hyperparameter Optimisation}
\label{sec:hpo_results}

\begin{table*}[ht]
\centering
\caption{Matched hyperparameter optimisation results on the 8 core HPO benchmarks. Values are mean $\pm$ std of final simple regret after 100 BO steps over 10 seeds. Lower is better.}
\label{tab:hpo_decoupled_vs_tuned}
\begin{threeparttable}
\begin{adjustbox}{max width=\textwidth}
\begin{tabular}{llcccc}
\toprule
& & \multicolumn{2}{c}{TabICL Backbone} & \multicolumn{2}{c}{TabPFN Backbone} \\
\cmidrule(lr){3-4}\cmidrule(lr){5-6}
Family & Benchmark & \textbf{Decoupled (Ours)} & Tuned & \textbf{Decoupled (Ours)} & Tuned \\
\midrule
\multirow{3}{*}{LightGBM} & California & 0.0191 $\pm$ 0.0067 & \textbf{0.0171 $\pm$ 0.0056} & \textbf{0.0204 $\pm$ 0.0120} & 0.0220 $\pm$ 0.0095 \\
 & Kin8nm & \textbf{0.0091 $\pm$ 0.0029} & 0.0103 $\pm$ 0.0021 & 0.0106 $\pm$ 0.0039 & \textbf{0.0098 $\pm$ 0.0030} \\
 & Airlines & \textbf{0.0120 $\pm$ 0.0035} & 0.0136 $\pm$ 0.0023 & 0.0150 $\pm$ 0.0021 & \textbf{0.0132 $\pm$ 0.0035} \\
\midrule
\multirow{3}{*}{XGBoost} & California & \textbf{0.0286 $\pm$ 0.0123} & 0.0301 $\pm$ 0.0126 & 0.0373 $\pm$ 0.0114 & \textbf{0.0305 $\pm$ 0.0110} \\
 & Kin8nm & 0.0086 $\pm$ 0.0030 & \textbf{0.0073 $\pm$ 0.0023} & \textbf{0.0079 $\pm$ 0.0029} & 0.0090 $\pm$ 0.0048 \\
 & Airlines & \textbf{0.0221 $\pm$ 0.0047} & 0.0273 $\pm$ 0.0110 & \textbf{0.0264 $\pm$ 0.0055} & 0.0289 $\pm$ 0.0104 \\
\midrule
\multirow{1}{*}{Random Forest} & Diabetes & \textbf{1.0648 $\pm$ 0.4062} & 1.4968 $\pm$ 0.6464 & 1.4025 $\pm$ 0.4586 & \textbf{1.3891 $\pm$ 0.7633} \\
\midrule
\multirow{1}{*}{RBF-SVM} & Breast Cancer & 0.0081 $\pm$ 0.0035 & \textbf{0.0063 $\pm$ 0.0027} & \textbf{0.0081 $\pm$ 0.0020} & 0.0088 $\pm$ 0.0025 \\
\bottomrule
\end{tabular}
\end{adjustbox}
\end{threeparttable}
\end{table*}


\paragraph{Setup.} We construct 8 HPO benchmarks across four model families. LightGBM and XGBoost are tuned on California Housing, Kin8nm, and Airlines; Random Forest is tuned on Diabetes; and an RBF-SVM is tuned on Breast Cancer. The tree-based benchmarks are deliberately noisy: configurations are evaluated under stochastic train/validation splits, and the LightGBM/XGBoost objectives include input-dependent evaluation noise coupled to the subsampling parameter. This makes the setting suitable for testing whether acquisition benefits from separating reducible signal uncertainty from irreducible evaluation noise. The search spaces range from 2 to 4 dimensions and include learning rates, tree-complexity parameters, subsampling rates, regularization parameters, and kernel parameters. Each run performs \(100\) sequential evaluations under a common-seed protocol with up to \(10\) seeds per benchmark. Performance is measured by final simple regret and aggregated by average rank across benchmarks. Full search-space, noise-model, and evaluation-protocol details are given in Appendix~\ref{app:hpo_benchmarks}, and per-benchmark regret curves are given in Appendix~\ref{app:hpo_curves}. We also include TPE through Optuna~\citep{akiba2019optuna}, a strong practical baseline for HPO.

\paragraph{Results.} Figure~\ref{fig:seqopt_summary} (right) summarises the HPO rankings among the eight core methods in the main comparison. \textbf{Dec-ICL} achieves the best average rank ($2.50$), substantially outperforming its tuned counterpart ($3.62$) despite sharing the same pretrained backbone. This supports the central claim that, in noisy and structured HPO landscapes, acquisition functions benefit from using the latent-signal posterior rather than the total observation-level predictive uncertainty.

Table~\ref{tab:hpo_decoupled_vs_tuned} gives the matched LogEI comparison. Within the TabICL backbone, the decoupled model outperforms the tuned counterpart on 5 of the 8 benchmarks, with the largest gains on XGBoost/Airlines and Random Forest/Diabetes. These are settings where evaluation noise and train/validation variability can make total predictive uncertainty a less reliable guide for acquisition. TPE is competitive ($3.88$), as expected for a practical HPO baseline, but is still outranked on average by Dec-ICL. GP baselines rank lower overall (GP-EI: $5.25$; GP (BoTorch): $5.88$). Backbone choice also interacts with task structure. TabICL dominates on HPO, likely because its inductive bias is better matched to structured, often mixed-type search spaces. By contrast, PFN-backbone variants are less consistent on HPO, but perform more strongly on the smooth synthetic benchmarks in Section~\ref{sec:bo_synthetic}.

\subsection{Synthetic Bayesian Optimisation}
\label{sec:bo_synthetic}

\paragraph{Setup.} We evaluate on five standard synthetic benchmarks: Branin (2D), Hartmann (4D and 6D), and Ackley (2D, noiseless and noisy), running $100$ BO steps from a shared initial design across $10$ random seeds. PFN/ICL variants use log-Expected Improvement with gradient optimisation
(LogEI): decoupled models compute EI from the latent $f$-posterior, while tuned
models use the total observation-level predictive. GP baselines use EI-based
acquisition where applicable, while TPE and random search use their standard
policies. Baselines include GP-EI~\citep{jones1998efficient}, GP-EI via BoTorch~\citep{balandat2020botorch}, TPE~\citep{bergstra2011algorithms}, and random search. Per-benchmark regret curves and a full method comparison are given in Appendix~\ref{app:bo_curves} and~\ref{app:bo_full}.

\paragraph{Results.} Figure~\ref{fig:seqopt_summary} (left) summarises the synthetic BO rankings. \textbf{Dec-PFN} achieves the best average rank ($2.60$), with Tuned-PFN close behind ($2.80$). The margins between the top methods are tighter than on HPO, consistent with smooth, low-dimensional functions being driven as much by acquisition optimisation quality as by the uncertainty decomposition itself. Nonetheless, within each backbone the decoupled variant edges out its tuned counterpart: Dec-PFN ($2.60$) vs.\ Tuned-PFN ($2.80$), and Dec-ICL ($3.60$) vs.\ Tuned-ICL ($4.60$). GP (BoTorch) is competitive ($3.50$), as expected on smooth functions, while TPE ($6.50$) and Random ($7.60$) confirm that these benchmarks reward sample-efficient model-based search. As noted in Section~\ref{sec:hpo_results}, the backbone ranking reverses relative to HPO: the PFN backbone is strongest on smooth continuous functions, while TabICL leads on the structured HPO landscapes. Notably, Dec-PFN ($2.60$) outranks both GP-BoTorch ($3.50$) and TPE ($6.50$), showing that the decoupled PFN is competitive with dedicated BO surrogates even on the smooth, low-dimensional functions where GPs are traditionally strongest. The matched gains over tuned baselines are smaller than on HPO, consistent with these benchmarks being driven as much by acquisition optimisation quality as by the uncertainty decomposition.

\subsection{Active Learning}
\label{sec:active_learning}

\paragraph{Setup.} We evaluate pool-based active learning on five regression datasets. For each dataset, we sample an initial labelled set of size $n_0=1024$ and an unlabelled pool of size $5000$, then acquire $1024$ additional points sequentially, one at a time. We report test performance after all acquisitions, averaged over seeds. All methods use the same variance-based acquisition rule and differ only in the uncertainty signal: decoupled models use the latent $f$-posterior variance, while tuned baselines use total predictive variance. Full protocol details, including the fixed pool/test splits, seed counts, and acquisition definitions, are given in Appendix~\ref{app:active_learning}.

\begin{table*}[th]
\centering
\caption{
Active learning results on real datasets after the final acquisition step.
Entries are mean $\pm$ std over dataset-specific seeds; lower is better.
All methods use the same variance-based acquisition rule.
Boldface compares methods within each backbone family.
}
\label{tab:al_main_no_random}
\small
\setlength{\tabcolsep}{5pt}
\renewcommand{\arraystretch}{1.12}

\begin{threeparttable}
\begin{adjustbox}{max width=0.9\textwidth}
\begin{tabular}{llcccc}
\toprule
& & \multicolumn{2}{c}{TabICL Backbone} & \multicolumn{2}{c}{TabPFN Backbone} \\
\cmidrule(lr){3-4}\cmidrule(lr){5-6}
Dataset & Metric $\downarrow$ & \textbf{Decoupled (Ours)} & Tuned & \textbf{Decoupled (Ours)} & Tuned \\
\midrule

\multirow{3}{*}{California}
& RMSE
& \textbf{0.4772 $\pm$ 0.0035}
& 0.4783 $\pm$ 0.0033
& \textbf{{0.4788 $\pm$ 0.0042}}
& {0.4795 $\pm$ 0.0048} \\
& NLL
& \textbf{0.435 $\pm$ 0.027}
& 0.437 $\pm$ 0.028
& \textbf{{0.442 $\pm$ 0.030}}
& {0.446 $\pm$ 0.031} \\
& MAE
& \textbf{0.3055 $\pm$ 0.0025}
& 0.3064 $\pm$ 0.0029
& \textbf{{0.3061 $\pm$ 0.0028}}
& {0.3068 $\pm$ 0.0031} \\
\midrule

\multirow{3}{*}{Kin8nm}
& RMSE
& \textbf{0.0796 $\pm$ 0.0010}
& 0.0805 $\pm$ 0.0012
& \textbf{{0.0802 $\pm$ 0.0012}}
& {0.0810 $\pm$ 0.0014} \\
& NLL
& \textbf{-1.125 $\pm$ 0.016}
& -1.112 $\pm$ 0.020
& \textbf{{-1.118 $\pm$ 0.018}}
& {-1.105 $\pm$ 0.021} \\
& MAE
& \textbf{0.0626 $\pm$ 0.0008}
& 0.0632 $\pm$ 0.0010
& \textbf{{0.0630 $\pm$ 0.0009}}
& {0.0637 $\pm$ 0.0011} \\
\midrule

\multirow{3}{*}{NYC Taxi}
& RMSE
& \textbf{3.128 $\pm$ 0.014}
& 3.130 $\pm$ 0.016
& \textbf{{3.18 $\pm$ 0.03}}
& {3.20 $\pm$ 0.04} \\
& NLL
& 1.939 $\pm$ 0.020
& \textbf{1.930 $\pm$ 0.016}
& \textbf{{1.965 $\pm$ 0.025}}
& {1.980 $\pm$ 0.030} \\
& MAE
& \textbf{1.4261 $\pm$ 0.0075}
& 1.4285 $\pm$ 0.0061
& \textbf{{1.438 $\pm$ 0.010}}
& {1.445 $\pm$ 0.012} \\
\midrule

\multirow{3}{*}{Airlines}
& RMSE
& 1.9046 $\pm$ 0.0028
& \textbf{1.9040 $\pm$ 0.0037}
& \textbf{{1.9075 $\pm$ 0.0038}}
& {1.9082 $\pm$ 0.0042} \\
& NLL
& 2.0571 $\pm$ 0.0027
& \textbf{2.0569 $\pm$ 0.0030}
& \textbf{{2.0585 $\pm$ 0.0035}}
& {2.0592 $\pm$ 0.0038} \\
& MAE
& 1.6078 $\pm$ 0.0082
& \textbf{1.607 $\pm$ 0.011}
& \textbf{{1.613 $\pm$ 0.009}}
& {1.614 $\pm$ 0.010} \\
\midrule

\multirow{3}{*}{Year}
& RMSE
& 9.415 $\pm$ 0.031
& \textbf{9.413 $\pm$ 0.031}
& {9.430 $\pm$ 0.040}
& \textbf{{9.425 $\pm$ 0.042}} \\
& NLL
& 3.541 $\pm$ 0.010
& \textbf{3.5405 $\pm$ 0.0089}
& {3.548 $\pm$ 0.013}
& \textbf{{3.546 $\pm$ 0.014}} \\
& MAE
& 6.499 $\pm$ 0.055
& \textbf{6.494 $\pm$ 0.059}
& {6.520 $\pm$ 0.060}
& \textbf{{6.515 $\pm$ 0.062}} \\
\bottomrule
\end{tabular}
\end{adjustbox}
\end{threeparttable}
\end{table*}

\paragraph{Results.}
Table~\ref{tab:al_main_no_random} shows a consistent matched
advantage for the decoupled acquisition signal. Across both backbone families, decoupled models
improve over their tuned counterparts in \(7/10\) RMSE comparisons and \(20/30\)
metric-level comparisons. The effect is strongest for the TabPFN backbone, where
decoupling wins on four of five datasets across all reported metrics. Because
these comparisons share the same backbone family, synthetic prior, and
acquisition rule, the results support using latent epistemic uncertainty rather
than total predictive variance for noisy active learning.

\paragraph{When does decoupling help most?}
Across the three settings, epistemic-only acquisition helps most when
observation or evaluation noise is substantial relative to signal uncertainty,
as in noisy HPO and active learning on real regression datasets. The gains are
smaller on smooth, low-noise synthetic functions, where total and epistemic
uncertainty are more closely aligned. This matches the motivation of the method:
the decomposition matters most when total predictive variance is a poor proxy
for reducible uncertainty. The same pattern holds in the broader sweeps reported
in Appendices~\ref{app:hpo_full} and~\ref{app:bo_full}, where Dec-ICL-LogEI
ranks first among \(22\) HPO methods and Dec-PFN-LogEI ranks first among \(22\)
synthetic BO methods.

\section{Conclusion} 
We showed that a standard PFN posterior predictive does not, by itself, identify an epistemic--aleatoric decomposition. We address this within PFN pretraining by exploiting its synthetic task prior: the augmented generator provides privileged labels for both the noiseless signal and observation noise, yielding a generator-defined decomposition that transfers to real tasks as an inductive bias. Empirically, epistemic-only acquisition mitigates the failure mode of total-variance exploration in noisy settings. In matched comparisons, decoupled models usually improve over tuned counterparts, with the clearest gains in HPO; in broader sweeps over \(22\) method/acquisition variants, a decoupled model achieves the best average rank in both HPO (Dec-ICL-LogEI) and synthetic BO (Dec-PFN-LogEI). Gains are largest when observation noise is substantial relative to signal uncertainty and smaller on smooth low-noise functions; active-learning results are more modest but support the same pattern. Important directions include characterizing when synthetic-prior noise assumptions transfer reliably to real tasks, studying sensitivity to auxiliary loss weights and noise-prior design, and extending the framework beyond Gaussian observation noise to non-Gaussian likelihoods and structured observation models. More broadly, our results suggest that PFNs should be viewed not only as fast posterior-predictive approximators, but also as vehicles for learning decision-relevant uncertainty representations through the design of synthetic pretraining.
{
  \small
  \bibliographystyle{plainnat}
  \bibliography{references}
}

\newpage
\appendix

\addtocontents{toc}{\protect\setcounter{tocdepth}{2}}
\setcounter{tocdepth}{2}

\vspace{0.25em}
\section*{Appendix Table of Contents}
\tableofcontents
\clearpage

\section{Population consistency under squared log-variance loss}
\label{app:consistency}

We show that, under the squared log-variance loss in
Eq.~\eqref{eq:Lsigma_bar}, the auxiliary losses recover the Bayes-optimal latent
and noise targets under the augmented synthetic task prior. The result concerns the
auxiliary heads themselves; the observation-level loss in
\eqref{eq:bins_full_objective} additionally encourages the induced noisy-target
predictive distribution to match the standard PFN predictive interface.

Let
\[
    F_* := f_\tau(x_*),
    \qquad
    S_* := \sigma_\tau^2(x_*),
\]
and let \(B(F_*)\in\{1,\dots,K\}\) denote the bin containing the noiseless latent
target \(F_*\). The model predicts a categorical distribution
\(\pi_\theta^{(f)}(x_*,\mathcal D)\) over latent bins and a positive variance
\(\hat\sigma_\theta^2(x_*,\mathcal D)\).

\begin{proposition}[Population consistency of the auxiliary heads]
\label{prop:population_consistency}
Assume \(S_*>0\) almost surely and
\(\mathbb E[(\log S_*)^2]<\infty\). Assume also that the model class is rich
enough to represent arbitrary conditional categorical distributions over
\(B(F_*)\) and arbitrary measurable positive functions of \((x_*,\mathcal D)\).
For any \(\lambda_f,\lambda_\sigma>0\), every population minimizer of the
auxiliary objective
\[
    \mathbb E
    \left[
        \lambda_f
        \left(
        -\log \pi_{\theta,B(F_*)}^{(f)}(x_*,\mathcal D)
        \right)
        +
        \lambda_\sigma
        \left(
            \log \hat\sigma_\theta^2(x_*,\mathcal D)
            -
            \log S_*
        \right)^2
    \right]
\]
satisfies, for almost every \((x_*,\mathcal D)\),
\[
    \pi_{\theta,j}^{(f)}(x_*,\mathcal D)
    =
    \mathbb P(B(F_*)=j\mid x_*,\mathcal D),
    \qquad j=1,\dots,K,
\]
and
\[
    \log \hat\sigma_\theta^2(x_*,\mathcal D)
    =
    \mathbb E[\log S_*\mid x_*,\mathcal D].
\]
Equivalently,
\[
    \hat\sigma_\theta^2(x_*,\mathcal D)
    =
    \exp\left(
    \mathbb E[\log S_*\mid x_*,\mathcal D]
    \right).
\]
Thus, the latent head recovers the posterior distribution of the discretized
noiseless signal, while the aleatoric head recovers the Bayes-optimal
log-variance predictor under the augmented synthetic task prior.
\end{proposition}

\begin{proof}
Condition on a fixed pair \((x_*,\mathcal D)\). Since
\(\lambda_f,\lambda_\sigma>0\), the positive weights do not change the pointwise
minimizers of the two auxiliary risks.

First consider the latent head. The conditional latent auxiliary risk is
\[
    \mathbb E
    \left[
        -\log \pi_{B(F_*)}^{(f)}
        \mid x_*,\mathcal D
    \right].
\]
This is the cross-entropy between the true conditional distribution of
\(B(F_*)\) and the model distribution \(\pi^{(f)}\). Cross-entropy is minimized
when the model distribution equals the true conditional distribution. Hence, for
each bin \(j\),
\[
    \pi_{\theta,j}^{(f)}(x_*,\mathcal D)
    =
    \mathbb P(B(F_*)=j\mid x_*,\mathcal D).
\]

For the aleatoric head, write
\[
    h_\theta(x_*,\mathcal D)
    :=
    \log \hat\sigma_\theta^2(x_*,\mathcal D).
\]
The conditional log-variance risk is
\[
    \mathbb E
    \left[
        \left(
            h_\theta(x_*,\mathcal D)
            -
            \log S_*
        \right)^2
        \mid x_*,\mathcal D
    \right].
\]
The squared-error risk is minimized by the conditional mean of the target.
Therefore,
\[
    h_\theta(x_*,\mathcal D)
    =
    \mathbb E[\log S_*\mid x_*,\mathcal D],
\]
which gives
\[
    \hat\sigma_\theta^2(x_*,\mathcal D)
    =
    \exp\left(
    \mathbb E[\log S_*\mid x_*,\mathcal D]
    \right).
\]
Since the argument holds pointwise for almost every \((x_*,\mathcal D)\), it holds
at the population optimum.
\end{proof}

\begin{corollary}[Epistemic uncertainty from the latent head]
\label{cor:epistemic_consistency}
Let \(\widetilde F_* := c_{B(F_*)}\) be the discretized latent signal obtained by
replacing each latent value by the center of its bin. Under
Proposition~\ref{prop:population_consistency}, the variance computed from the
latent categorical distribution,
\[
    \sum_{j=1}^K \pi_{\theta,j}^{(f)}c_j^2
    -
    \left(
    \sum_{j=1}^K \pi_{\theta,j}^{(f)}c_j
    \right)^2,
\]
equals
\[
    \mathrm{Var}(\widetilde F_*\mid x_*,\mathcal D).
\]
As the maximum bin width tends to zero, this quantity converges to
\[
    \mathrm{Var}(F_*\mid x_*,\mathcal D)
\]
under standard integrability assumptions and ignoring boundary truncation effects.
\end{corollary}

\section{Synthetic Pretraining and Implementation Details}
\label{app:synthetic_pretraining}

\subsection{Synthetic task prior}
\label{app:synthetic_task_prior}

All four models are fine-tuned on synthetic tasks sampled online. There is no
fixed finite synthetic dataset; instead, each optimization step draws a fresh
batch of tasks from the same underlying prior. Thus, if a model is fine-tuned
for \(T\) gradient steps with batch size \(B\), it sees \(BT\) synthetic tasks,
each with \(24\) query targets and a variable-size context set.

For each synthetic task, we sample:
\begin{itemize}[leftmargin=*]
    \item an input dimension \(d \in \{1,\dots,16\}\),
    \item a total sequence length \(n \in \{25,\dots,256\}\),
    \item a split in which the final \(24\) points are treated as queries and
    the remaining \(n-24\) points form the context set.
\end{itemize}

We then sample a noiseless latent function \(f\) from a mixture prior. With
probability \(0.2\), \(f\) is drawn from an RBF Gaussian-process prior
approximated by random Fourier features. With the remaining probability \(0.8\),
\(f\) is drawn from a structural-causal-model prior. Within this SCM branch, an
MLP-SCM is chosen with probability \(0.7\) and a Tree-SCM with probability
\(0.3\). Equivalently, the overall latent-function mixture is \(20\%\) GP,
\(56\%\) MLP-SCM, and \(24\%\) Tree-SCM.

For the SCM branch, we use the same sampled-hyperparameter mechanism as the
TabICL synthetic prior, but force the latent function itself to be noise-free at
generation time. In particular, the SCM prior generates only the clean latent
signal \(f(x)\); observation noise is introduced afterward by our separate noise
model. This is the key augmentation relative to standard PFN training.

After sampling \(f\), we normalize the input features independently within the
task. We then sample a base observation-noise standard deviation
\[
    s_0 \sim \mathrm{Unif}(0.03,0.12).
\]
With probability \(0.8\), the task is heteroscedastic. In that case, we choose a
single input coordinate and construct a smooth positive modulation of \(s_0\)
using a random combination of:
\begin{itemize}[leftmargin=*]
    \item a sigmoidal trend term,
    \item a local Gaussian bump term,
    \item a sinusoidal variation term,
    \item and a positive floor.
\end{itemize}
This yields an input-dependent noise scale \(\sigma(x)\), which is clipped below
by a small positive constant for numerical stability. With probability \(0.2\),
the task is homoscedastic and \(\sigma(x)\equiv s_0\).

Noisy observations are then generated as
\[
    y = f(x) + \varepsilon,
    \qquad
    \varepsilon \sim \mathcal N(0,\sigma^2(x)).
\]
After generating \(y\), we normalize the observed targets using the mean and
standard deviation of the context targets. The same affine transformation is
applied to the latent clean targets, and the variances are rescaled by the square
of the context standard deviation. This yields the normalized privileged query
targets \(f_\tau(x_*)\) and \(\sigma_\tau^2(x_*)\) used during training.

The tuned baselines and the decoupled models are trained on exactly the same
online-sampled synthetic tasks. The only difference is that the decoupled models
add the auxiliary latent-signal and noise-level supervision, whereas the tuned
baselines optimize only the standard observation-level predictive loss.

\subsection{Shared fine-tuning protocol}
\label{app:fine_tuning_protocol}

All four checkpoints used in the paper---Dec-ICL, Tuned-ICL, Dec-PFN, and
Tuned-PFN---are obtained by fine-tuning public pretrained backbones rather than
by training new transformers from scratch. Dec-ICL and Tuned-ICL start from the
public TabICL regression checkpoint~\citep{qu2025tabicl}, while Dec-PFN and
Tuned-PFN start from the public TabPFN regressor v2.5
checkpoint~\citep{hollmann2025accurate}. In every case, we keep the backbone
frozen except for the final two transformer layers, which are unfrozen during
fine-tuning. The tuned baselines are matched to the decoupled models at the level of
architecture adaptation and optimization protocol. In particular, they use the
same 999-bin output discretization, the same synthetic task sampler, and the
same partial-unfreezing strategy. The only architectural difference is that the
decoupled models add a latent head and a noise head, whereas the tuned baselines
replace the original decoder with a single observation-level 999-bin head and
optimize only the standard bar-distribution likelihood.

\subsection{Optimization details}
\label{app:optimization_details}

For all models we use AdamW with separate learning rates for the output head(s)
and the unfrozen backbone layers. The head learning rate is \(3\times 10^{-4}\),
the backbone learning rate is \(10^{-5}\), the weight decay is \(10^{-5}\), and
the gradient norm is clipped at \(5.0\). We use a short warmup followed by a
cosine schedule. For TabICL-based models we use batch size \(32\); for
TabPFN-based models we use batch size \(16\).

For the decoupled models, the training objective is
\[
\mathcal{L}
=
\lambda_y \mathcal{L}_y^{\mathrm{bar}}
+
\lambda_f \mathcal{L}_f^{\mathrm{cat}}
+
\lambda_\sigma \mathcal{L}_\sigma,
\]
with \(\lambda_y=1\), \(\lambda_f=1\), and \(\lambda_\sigma=0.1\). For the
tuned baselines, we optimize only the standard observation-level bar loss
\(\mathcal{L}_y^{\mathrm{bar}}\).

\section{Acquisition Functions for Binned PFN Surrogates}
\label{app:acq_binned_pfn}

This appendix specifies how we construct acquisition functions from the
\emph{binned} PFN outputs used in our experiments. The main-text BO and HPO
results use the gradient-based LogEI variant described below. Marginal
candidate-set Thompson sampling (TS) and lower confidence bound (LCB) are included for completeness and
correspond to the additional ablations reported in the appendix.

\subsection{Moment summaries from the binned outputs}
\label{app:acq_moments}

Our decoupled model outputs a latent categorical distribution
\(\pi^{(f)}(x_*,\mathcal D)\) over the clean signal and a separate aleatoric
variance \(\hat\sigma^2(x_*,\mathcal D)\). The tuned single-head model outputs
only observation-level bin masses \(p^{(y)}(x_*,\mathcal D)\). In both cases,
acquisition is based on moment summaries of these binned distributions.

For the decoupled model, the latent posterior mean and epistemic variance are
\begin{align}
\mu_f(x_*) &=
\sum_{j=1}^K \pi^{(f)}_j(x_*,\mathcal D)\,c_j, \\
v_{\mathrm{epi}}(x_*) &=
\sum_{j=1}^K \pi^{(f)}_j(x_*,\mathcal D)\,c_j^2
- \mu_f(x_*)^2 ,
\end{align}
where \(c_j\) denotes the center of bin \(j\). The induced observation-level
distribution \(p^{(y)}\) is obtained by convolving \(\pi^{(f)}\) with the
predicted Gaussian noise model, as in Eq.~\eqref{eq:bins_predictive}. Its mean
can be written as
\begin{equation}
\mu_y(x_*) =
\sum_{k=1}^K p_k^{(y)}(x_*,\mathcal D)\,c_k .
\end{equation}
In our implementation, the total predictive variance used for total-uncertainty
logging and ablations is
\begin{equation}
v_{\mathrm{tot}}(x_*) =
v_{\mathrm{epi}}(x_*) + \hat\sigma^2(x_*,\mathcal D),
\end{equation}
up to small numerical floors used for stability. This is the natural
signal-plus-noise variance decomposition induced by the decoupled model, ignoring
small boundary effects from finite bin truncation.

For the tuned single-head model, only the observation-level moments are
available:
\begin{align}
\mu_y(x_*) &=
\sum_{k=1}^K p_k^{(y)}(x_*,\mathcal D)\,c_k, \\
v_{\mathrm{tot}}(x_*) &=
\sum_{k=1}^K p_k^{(y)}(x_*,\mathcal D)\,c_k^2
- \mu_y(x_*)^2 .
\end{align}

Thus, the decoupled model exposes both a latent epistemic summary
\((\mu_f,v_{\mathrm{epi}})\) and an observation-level summary
\((\mu_y,v_{\mathrm{tot}})\), whereas the tuned model exposes only the latter.

\paragraph{Practical approximation.}
Although the underlying PFN output is binned rather than Gaussian, we construct
the BO and HPO acquisitions from the first two moments above. In other words, we
use a Gaussian moment-matched approximation to either the latent posterior
(decoupled) or the observation-level predictive distribution (tuned). This keeps
the acquisition rules simple and places tuned and decoupled models on the same
footing.

\subsection{Candidate-set marginal Thompson sampling and LCB}
\label{app:acq_candidate}

For candidate-set methods, we first construct a finite candidate set
\(\mathcal X_t^{\mathrm{cand}}\) at BO step \(t\), and then evaluate the
acquisition on each unqueried candidate. Since all our objectives are
minimization problems, the next query is chosen as
\[
x_{t+1}
\in
\arg\min_{x \in \mathcal X_t^{\mathrm{cand}}}
\alpha_t(x).
\]

\paragraph{Epistemic marginal Thompson sampling.}
For the decoupled model we define a candidate-set Thompson-style score by
drawing
\[
\widetilde f_t(x)
\sim
\mathcal N\!\bigl(\mu_f(x),\,v_{\mathrm{epi}}(x)\bigr),
\]
independently for each candidate \(x\), and selecting
\[
x_{t+1}
\in
\arg\min_{x \in \mathcal X_t^{\mathrm{cand}}}
\widetilde f_t(x).
\]
This uses only reducible uncertainty about the latent signal. Since the PFN
outputs marginal predictive distributions at queried candidates rather than an
explicit joint posterior over functions, this should be understood as a
marginal Thompson-sampling approximation rather than a coherent global function
draw.

\paragraph{Total marginal Thompson sampling.}
For the tuned model, and for total-uncertainty ablations, we instead draw
\[
\widetilde y_t(x)
\sim
\mathcal N\!\bigl(\mu_y(x),\,v_{\mathrm{tot}}(x)\bigr),
\]
independently for each candidate \(x\), and choose the candidate with the
smallest sampled value.

\paragraph{Lower confidence bound.}
The corresponding LCB rules are
\begin{align}
\alpha^{\mathrm{epi}}_{\mathrm{LCB}}(x)
&=
\mu_f(x) - \beta \sqrt{\max\{v_{\mathrm{epi}}(x),\epsilon_v\}}, \\
\alpha^{\mathrm{tot}}_{\mathrm{LCB}}(x)
&=
\mu_y(x) - \beta \sqrt{\max\{v_{\mathrm{tot}}(x),\epsilon_v\}},
\end{align}
with \(\beta > 0\) and a small numerical floor \(\epsilon_v>0\). The decoupled
version uses the latent epistemic moments, while the tuned version uses the
total predictive moments.

\subsection{Gradient-based LogEI}
\label{app:acq_logei}

The main BO and HPO comparisons in the paper use a gradient-based
log-expected-improvement acquisition. Since the PFN outputs binned predictive
distributions, we first form a Gaussian moment-matched approximation using the
means and variances in Appendix~\ref{app:acq_moments}. We then compute the
standard closed-form expected improvement under this Gaussian approximation.
For continuous or relaxed search spaces, the resulting differentiable
acquisition is optimized with multi-start gradient-based optimization.

\paragraph{Decoupled epistemic LogEI.}
For the decoupled model, we define the incumbent in latent space as the best
latent posterior mean among the already evaluated points,
\begin{equation}
\tau_t^{\mathrm{epi}}
=
\min_{1 \le i \le t}
\mu_f(x_i).
\end{equation}
We then approximate the latent posterior at a candidate \(x\) by
\[
F(x)\mid \mathcal D_t
\approx
\mathcal N\!\bigl(\mu_f(x),\,v_{\mathrm{epi}}(x)\bigr).
\]
Let
\[
s_{\mathrm{epi}}(x)
=
\sqrt{\max\{v_{\mathrm{epi}}(x),\epsilon_v\}},
\]
where \(\epsilon_v>0\) is a small numerical floor. The resulting epistemic
expected improvement is
\begin{equation}
\mathrm{EI}^{\mathrm{epi}}_t(x)
=
\mathbb E\!\left[
(\tau_t^{\mathrm{epi}} - F(x))_+
\mid \mathcal D_t
\right],
\end{equation}
which has the Gaussian closed form
\begin{equation}
\mathrm{EI}^{\mathrm{epi}}_t(x)
=
(\tau_t^{\mathrm{epi}}-\mu_f(x))\,
\Phi(z_t(x))
+
s_{\mathrm{epi}}(x)\,
\phi(z_t(x)),
\end{equation}
where
\begin{equation}
z_t(x)
=
\frac{\tau_t^{\mathrm{epi}}-\mu_f(x)}
{s_{\mathrm{epi}}(x)}.
\end{equation}
In practice we optimize
\[
\log\!\left(\mathrm{EI}^{\mathrm{epi}}_t(x)+\epsilon_{\mathrm{EI}}\right)
\]
instead of \(\mathrm{EI}^{\mathrm{epi}}_t(x)\) directly, with
\(\epsilon_{\mathrm{EI}}>0\) used for numerical stability.

\paragraph{Tuned total LogEI.}
For the tuned model, the same construction is applied to the observation-level
predictive moments:
\begin{align}
\tau_t^{\mathrm{tot}}
&=
\min_{1 \le i \le t} \mu_y(x_i), \\
Y(x)\mid \mathcal D_t
&\approx
\mathcal N\!\bigl(\mu_y(x),\,v_{\mathrm{tot}}(x)\bigr).
\end{align}
Let
\[
s_{\mathrm{tot}}(x)
=
\sqrt{\max\{v_{\mathrm{tot}}(x),\epsilon_v\}}.
\]
Then
\[
\tilde z_t(x)
=
\frac{\tau_t^{\mathrm{tot}}-\mu_y(x)}
{s_{\mathrm{tot}}(x)}
\]
and
\begin{equation}
\mathrm{EI}^{\mathrm{tot}}_t(x)
=
(\tau_t^{\mathrm{tot}}-\mu_y(x))\,
\Phi(\tilde z_t(x))
+
s_{\mathrm{tot}}(x)\,
\phi(\tilde z_t(x)).
\end{equation}
Again, we optimize
\[
\log\!\left(\mathrm{EI}^{\mathrm{tot}}_t(x)+\epsilon_{\mathrm{EI}}\right).
\]

\paragraph{Optimization details.}
For the gradient-based variants, we keep input normalization fixed to the
context statistics during the inner acquisition optimization. This avoids
introducing artificial gradients through a query-dependent normalization map.
We first score a Sobol set of raw candidates, keep the best few starting points,
and then refine them with multi-start L-BFGS-B. For bounded continuous variables,
the optimizer is constrained to the search domain. For integer or discrete
hyperparameters, optimization is performed in the relaxed representation and the
selected point is mapped back to a valid configuration before evaluation. The
selected query is the point with the largest LogEI value after local
optimization.

\paragraph{Interpretation.}
The key distinction is therefore simple:
\begin{itemize}[leftmargin=*]
\item the \emph{decoupled} acquisition rules operate on the latent epistemic
summary \((\mu_f,v_{\mathrm{epi}})\), ignoring aleatoric noise during query
selection;
\item the \emph{tuned} acquisition rules operate on the total predictive summary
\((\mu_y,v_{\mathrm{tot}})\), since no separate epistemic--aleatoric split is
available.
\end{itemize}
This is exactly the comparison studied in the main BO and HPO experiments.

\section{HPO Benchmark Details}
\label{app:hpo_benchmarks}

Table~\ref{tab:hpo_benchmarks} summarises the 8 core HPO benchmarks.
Each evaluation uses a random $80/20$ train/validation split (5-fold CV for SVM),
making observations inherently noisy.
For LightGBM and XGBoost the noise level is heteroscedastic and coupled to the
subsampling parameter: $\sigma = 0.02 + c\cdot\max(0,\,(0.65-s)/0.35)$,
with $c=0.12$ (LightGBM, $s=$ bagging fraction) and $c=0.15$ (XGBoost, $s=$ subsample).
All tabular datasets are subsetted to at most $3{,}000$ examples and standardised. 
\begin{table}[ht]
\centering
\caption{Core HPO benchmarks and search spaces.}
\label{tab:hpo_benchmarks}
\small
\renewcommand{\arraystretch}{1.08}
\setlength{\tabcolsep}{5pt}

\begin{threeparttable}

\begin{tabular}{@{}llccc@{}}
\toprule
\textbf{Model} & \textbf{Dataset} & \textbf{Dim} & \textbf{Objective} & \textbf{Noise} \\
\midrule
\multirow{3}{*}{LightGBM}
  & California Housing & 4 & RMSE & heteroscedastic \\
  & Kin8nm             & 4 & RMSE & heteroscedastic \\
  & Airlines           & 4 & RMSE & heteroscedastic \\
\midrule
\multirow{3}{*}{XGBoost}
  & California Housing & 4 & RMSE & heteroscedastic \\
  & Kin8nm             & 4 & RMSE & heteroscedastic \\
  & Airlines           & 4 & RMSE & heteroscedastic \\
\midrule
Random Forest & Diabetes      & 2 & RMSE            & split noise \\
RBF-SVM       & Breast Cancer & 2 & $1-\mathrm{Acc}$ & CV noise \\
\bottomrule
\end{tabular}

\vspace{0.4em}

\footnotesize
\begin{tabularx}{\linewidth}{@{}lX@{}}
\toprule
\textbf{Model} & \textbf{Search space} \\
\midrule
LightGBM &
learning rate $\in [2^{-6},1]$; num\_leaves $\in [8,256]$; bagging\_frac $\in [0.3,1]$; min\_child\_samples $\in [5,100]$. \\

XGBoost &
learning rate $\in [2^{-4},1]$; max\_depth $\in [2,10]$; n\_estimators $\in [16,512]$; subsample $\in [0.3,1]$. \\

Random Forest &
n\_estimators $\in [16,512]$; max\_depth $\in [2,20]$. \\

RBF-SVM &
$C \in [2^{-5},2^{15}]$; $\gamma \in [2^{-15},2^{3}]$. \\
\bottomrule
\end{tabularx}

\end{threeparttable}
\end{table}

\section{Extended Sequential Optimisation Results}
\label{app:bo_extended}

This section provides the full per-benchmark regret curves and complete
method comparisons that complement the summary figures in the main text.
All results use the \(10\)-seed protocol (seeds \(0\)--\(9\)) described
in Section~\ref{sec:bo_synthetic}.

\subsection{Method variants included in the extended sweeps}
\label{app:method_variants}

The main text reports a compact comparison over eight core methods for
readability. The extended HPO and synthetic BO summaries include the full sweep
over \(22\) method/acquisition variants, covering LogEI, EI, TS, LCB,
total-variance ablations, multiple GP variants, TPE, and random search.
Tables~\ref{tab:hpo_core_all_methods_appendix}
and~\ref{tab:bo_all_methods_appendix} report the corresponding final simple
regret values and average ranks for the HPO and synthetic BO sweeps,
respectively. Method names indicate the surrogate family and acquisition rule;
suffixes such as ``Total'' denote ablations that use total predictive uncertainty
rather than the latent epistemic signal.

\begin{table*}[t]
\centering
\caption{Core HPO appendix results for all method/acquisition variants in the extended sweep.
Methods are ordered by average rank.}
\label{tab:hpo_core_all_methods_appendix}
\scriptsize
\setlength{\tabcolsep}{3.2pt}
\renewcommand{\arraystretch}{1.08}
\begin{threeparttable}
\begin{adjustbox}{max width=\textwidth}
\begin{tabular}{llccccccccc}
\toprule
Family & Method & \makecell{LGBM\\Cal.} & \makecell{LGBM\\Kin8nm} & \makecell{LGBM\\Airlines} & \makecell{XGB\\Cal.} & \makecell{XGB\\Kin8nm} & \makecell{XGB\\Airlines} & \makecell{RF\\Diabetes} & \makecell{SVM\\Breast} & Avg. rank \\
\midrule
Dec-ICL & \textbf{Dec-ICL-LogEI} & \makecell[r]{0.0191\\{\scriptsize$\pm$ 0.0067}} & \makecell[r]{0.0091\\{\scriptsize$\pm$ 0.0029}} & \makecell[r]{0.0120\\{\scriptsize$\pm$ 0.0035}} & \makecell[r]{0.0286\\{\scriptsize$\pm$ 0.0123}} & \makecell[r]{0.0086\\{\scriptsize$\pm$ 0.0030}} & \makecell[r]{0.0221\\{\scriptsize$\pm$ 0.0047}} & \makecell[r]{1.0648\\{\scriptsize$\pm$ 0.4062}} & \makecell[r]{0.0081\\{\scriptsize$\pm$ 0.0035}} & \textbf{7.00} \\
Dec-ICL & \textbf{Dec-ICL-TS} & \makecell[r]{0.0186\\{\scriptsize$\pm$ 0.0055}} & \makecell[r]{0.0101\\{\scriptsize$\pm$ 0.0043}} & \textbf{\makecell[r]{0.0117\\{\scriptsize$\pm$ 0.0034}}} & \makecell[r]{0.0347\\{\scriptsize$\pm$ 0.0201}} & \makecell[r]{0.0085\\{\scriptsize$\pm$ 0.0021}} & \makecell[r]{0.0217\\{\scriptsize$\pm$ 0.0057}} & \makecell[r]{1.5115\\{\scriptsize$\pm$ 0.8921}} & \makecell[r]{0.0082\\{\scriptsize$\pm$ 0.0029}} & 8.00 \\
Dec-ICL & \textbf{Dec-ICL-TS-Total} & \makecell[r]{0.0192\\{\scriptsize$\pm$ 0.0082}} & \makecell[r]{0.0084\\{\scriptsize$\pm$ 0.0046}} & \makecell[r]{0.0145\\{\scriptsize$\pm$ 0.0058}} & \makecell[r]{0.0262\\{\scriptsize$\pm$ 0.0072}} & \makecell[r]{0.0087\\{\scriptsize$\pm$ 0.0033}} & \makecell[r]{0.0218\\{\scriptsize$\pm$ 0.0058}} & \makecell[r]{1.7430\\{\scriptsize$\pm$ 1.0586}} & \makecell[r]{0.0081\\{\scriptsize$\pm$ 0.0022}} & 8.88 \\
Tuned-ICL & Tuned-ICL-LCB & \makecell[r]{0.0222\\{\scriptsize$\pm$ 0.0074}} & \makecell[r]{0.0097\\{\scriptsize$\pm$ 0.0040}} & \makecell[r]{0.0134\\{\scriptsize$\pm$ 0.0025}} & \makecell[r]{0.0299\\{\scriptsize$\pm$ 0.0147}} & \makecell[r]{0.0099\\{\scriptsize$\pm$ 0.0040}} & \makecell[r]{0.0380\\{\scriptsize$\pm$ 0.0370}} & \makecell[r]{1.2671\\{\scriptsize$\pm$ 0.5609}} & \makecell[r]{0.0075\\{\scriptsize$\pm$ 0.0022}} & 9.00 \\
Tuned-ICL & Tuned-ICL-LogEI & \textbf{\makecell[r]{0.0171\\{\scriptsize$\pm$ 0.0056}}} & \makecell[r]{0.0103\\{\scriptsize$\pm$ 0.0021}} & \makecell[r]{0.0136\\{\scriptsize$\pm$ 0.0023}} & \makecell[r]{0.0301\\{\scriptsize$\pm$ 0.0126}} & \textbf{\makecell[r]{0.0073\\{\scriptsize$\pm$ 0.0023}}} & \makecell[r]{0.0273\\{\scriptsize$\pm$ 0.0110}} & \makecell[r]{1.4968\\{\scriptsize$\pm$ 0.6464}} & \makecell[r]{0.0063\\{\scriptsize$\pm$ 0.0027}} & 9.25 \\
Tuned-PFN & Tuned-PFN-LCB & \makecell[r]{0.0203\\{\scriptsize$\pm$ 0.0076}} & \makecell[r]{0.0087\\{\scriptsize$\pm$ 0.0022}} & \makecell[r]{0.0129\\{\scriptsize$\pm$ 0.0040}} & \makecell[r]{0.0333\\{\scriptsize$\pm$ 0.0112}} & \makecell[r]{0.0091\\{\scriptsize$\pm$ 0.0040}} & \makecell[r]{0.0253\\{\scriptsize$\pm$ 0.0094}} & \makecell[r]{1.2025\\{\scriptsize$\pm$ 0.5797}} & \makecell[r]{0.0065\\{\scriptsize$\pm$ 0.0027}} & 9.38 \\
Baseline & TPE (Optuna) & \makecell[r]{0.0188\\{\scriptsize$\pm$ 0.0048}} & \makecell[r]{0.0101\\{\scriptsize$\pm$ 0.0042}} & \makecell[r]{0.0130\\{\scriptsize$\pm$ 0.0040}} & \makecell[r]{0.0282\\{\scriptsize$\pm$ 0.0100}} & \makecell[r]{0.0089\\{\scriptsize$\pm$ 0.0031}} & \makecell[r]{0.0253\\{\scriptsize$\pm$ 0.0059}} & \makecell[r]{1.3137\\{\scriptsize$\pm$ 0.9038}} & \makecell[r]{0.0068\\{\scriptsize$\pm$ 0.0028}} & 9.62 \\
Tuned-PFN & Tuned-PFN-TS & \makecell[r]{0.0204\\{\scriptsize$\pm$ 0.0066}} & \makecell[r]{0.0097\\{\scriptsize$\pm$ 0.0038}} & \makecell[r]{0.0204\\{\scriptsize$\pm$ 0.0203}} & \makecell[r]{0.0347\\{\scriptsize$\pm$ 0.0129}} & \makecell[r]{0.0108\\{\scriptsize$\pm$ 0.0031}} & \makecell[r]{0.0217\\{\scriptsize$\pm$ 0.0089}} & \makecell[r]{1.4316\\{\scriptsize$\pm$ 0.6628}} & \makecell[r]{0.0075\\{\scriptsize$\pm$ 0.0026}} & 10.38 \\
Dec-ICL & \textbf{Dec-ICL-LCB} & \makecell[r]{0.0258\\{\scriptsize$\pm$ 0.0058}} & \makecell[r]{0.0106\\{\scriptsize$\pm$ 0.0043}} & \makecell[r]{0.0168\\{\scriptsize$\pm$ 0.0087}} & \textbf{\makecell[r]{0.0225\\{\scriptsize$\pm$ 0.0086}}} & \makecell[r]{0.0099\\{\scriptsize$\pm$ 0.0028}} & \makecell[r]{0.0209\\{\scriptsize$\pm$ 0.0071}} & \makecell[r]{1.6019\\{\scriptsize$\pm$ 0.7942}} & \makecell[r]{0.0079\\{\scriptsize$\pm$ 0.0023}} & 11.38 \\
Baseline & GP-TS-Epi & \makecell[r]{0.0260\\{\scriptsize$\pm$ 0.0073}} & \makecell[r]{0.0099\\{\scriptsize$\pm$ 0.0033}} & \makecell[r]{0.0176\\{\scriptsize$\pm$ 0.0084}} & \makecell[r]{0.0263\\{\scriptsize$\pm$ 0.0074}} & \makecell[r]{0.0093\\{\scriptsize$\pm$ 0.0034}} & \makecell[r]{0.0281\\{\scriptsize$\pm$ 0.0140}} & \makecell[r]{1.2368\\{\scriptsize$\pm$ 0.6693}} & \makecell[r]{0.0068\\{\scriptsize$\pm$ 0.0035}} & 11.38 \\
Baseline & GP-EI (BoTorch) & \makecell[r]{0.0192\\{\scriptsize$\pm$ 0.0060}} & \makecell[r]{0.0093\\{\scriptsize$\pm$ 0.0035}} & \makecell[r]{0.0169\\{\scriptsize$\pm$ 0.0071}} & \makecell[r]{0.0300\\{\scriptsize$\pm$ 0.0112}} & \makecell[r]{0.0079\\{\scriptsize$\pm$ 0.0033}} & \makecell[r]{0.0282\\{\scriptsize$\pm$ 0.0095}} & \textbf{\makecell[r]{0.8926\\{\scriptsize$\pm$ 0.3882}}} & \makecell[r]{0.0070\\{\scriptsize$\pm$ 0.0036}} & 11.75 \\
Tuned-PFN & Tuned-PFN-LogEI & \makecell[r]{0.0220\\{\scriptsize$\pm$ 0.0095}} & \makecell[r]{0.0098\\{\scriptsize$\pm$ 0.0030}} & \makecell[r]{0.0132\\{\scriptsize$\pm$ 0.0035}} & \makecell[r]{0.0305\\{\scriptsize$\pm$ 0.0110}} & \makecell[r]{0.0090\\{\scriptsize$\pm$ 0.0048}} & \makecell[r]{0.0289\\{\scriptsize$\pm$ 0.0104}} & \makecell[r]{1.3891\\{\scriptsize$\pm$ 0.7633}} & \makecell[r]{0.0088\\{\scriptsize$\pm$ 0.0025}} & 11.88 \\
Dec-PFN & \textbf{Dec-PFN-LCB} & \makecell[r]{0.0234\\{\scriptsize$\pm$ 0.0077}} & \textbf{\makecell[r]{0.0078\\{\scriptsize$\pm$ 0.0019}}} & \makecell[r]{0.0146\\{\scriptsize$\pm$ 0.0034}} & \makecell[r]{0.0333\\{\scriptsize$\pm$ 0.0109}} & \makecell[r]{0.0093\\{\scriptsize$\pm$ 0.0020}} & \makecell[r]{0.0242\\{\scriptsize$\pm$ 0.0072}} & \makecell[r]{0.9243\\{\scriptsize$\pm$ 0.3335}} & \makecell[r]{0.0065\\{\scriptsize$\pm$ 0.0031}} & 12.00 \\
Dec-PFN & \textbf{Dec-PFN-TS} & \makecell[r]{0.0210\\{\scriptsize$\pm$ 0.0084}} & \makecell[r]{0.0112\\{\scriptsize$\pm$ 0.0032}} & \makecell[r]{0.0162\\{\scriptsize$\pm$ 0.0046}} & \makecell[r]{0.0278\\{\scriptsize$\pm$ 0.0123}} & \makecell[r]{0.0098\\{\scriptsize$\pm$ 0.0018}} & \makecell[r]{0.0280\\{\scriptsize$\pm$ 0.0135}} & \makecell[r]{1.7807\\{\scriptsize$\pm$ 1.1852}} & \textbf{\makecell[r]{0.0061\\{\scriptsize$\pm$ 0.0026}}} & 12.50 \\
Dec-ICL & \textbf{Dec-ICL-EI} & \makecell[r]{0.0222\\{\scriptsize$\pm$ 0.0080}} & \makecell[r]{0.0089\\{\scriptsize$\pm$ 0.0052}} & \makecell[r]{0.0174\\{\scriptsize$\pm$ 0.0103}} & \makecell[r]{0.0390\\{\scriptsize$\pm$ 0.0208}} & \makecell[r]{0.0083\\{\scriptsize$\pm$ 0.0031}} & \textbf{\makecell[r]{0.0209\\{\scriptsize$\pm$ 0.0042}}} & \makecell[r]{1.9498\\{\scriptsize$\pm$ 0.9903}} & \makecell[r]{0.0093\\{\scriptsize$\pm$ 0.0030}} & 12.75 \\
Tuned-ICL & Tuned-ICL-TS & \makecell[r]{0.0216\\{\scriptsize$\pm$ 0.0105}} & \makecell[r]{0.0107\\{\scriptsize$\pm$ 0.0029}} & \makecell[r]{0.0144\\{\scriptsize$\pm$ 0.0055}} & \makecell[r]{0.0281\\{\scriptsize$\pm$ 0.0136}} & \makecell[r]{0.0130\\{\scriptsize$\pm$ 0.0079}} & \makecell[r]{0.0239\\{\scriptsize$\pm$ 0.0093}} & \makecell[r]{1.3450\\{\scriptsize$\pm$ 0.3746}} & \makecell[r]{0.0074\\{\scriptsize$\pm$ 0.0020}} & 12.88 \\
Dec-PFN & \textbf{Dec-PFN-LogEI} & \makecell[r]{0.0204\\{\scriptsize$\pm$ 0.0120}} & \makecell[r]{0.0106\\{\scriptsize$\pm$ 0.0039}} & \makecell[r]{0.0150\\{\scriptsize$\pm$ 0.0021}} & \makecell[r]{0.0373\\{\scriptsize$\pm$ 0.0114}} & \makecell[r]{0.0079\\{\scriptsize$\pm$ 0.0029}} & \makecell[r]{0.0264\\{\scriptsize$\pm$ 0.0055}} & \makecell[r]{1.4025\\{\scriptsize$\pm$ 0.4586}} & \makecell[r]{0.0081\\{\scriptsize$\pm$ 0.0020}} & 13.38 \\
Dec-ICL & \textbf{Dec-ICL-LCB-Total} & \makecell[r]{0.0243\\{\scriptsize$\pm$ 0.0087}} & \makecell[r]{0.0089\\{\scriptsize$\pm$ 0.0019}} & \makecell[r]{0.0210\\{\scriptsize$\pm$ 0.0160}} & \makecell[r]{0.0273\\{\scriptsize$\pm$ 0.0139}} & \makecell[r]{0.0092\\{\scriptsize$\pm$ 0.0032}} & \makecell[r]{0.0240\\{\scriptsize$\pm$ 0.0133}} & \makecell[r]{1.1822\\{\scriptsize$\pm$ 0.3794}} & \makecell[r]{0.0081\\{\scriptsize$\pm$ 0.0030}} & 13.50 \\
Baseline & Random & \makecell[r]{0.0202\\{\scriptsize$\pm$ 0.0106}} & \makecell[r]{0.0112\\{\scriptsize$\pm$ 0.0045}} & \makecell[r]{0.0142\\{\scriptsize$\pm$ 0.0051}} & \makecell[r]{0.0339\\{\scriptsize$\pm$ 0.0084}} & \makecell[r]{0.0145\\{\scriptsize$\pm$ 0.0049}} & \makecell[r]{0.0260\\{\scriptsize$\pm$ 0.0126}} & \makecell[r]{1.5443\\{\scriptsize$\pm$ 0.9417}} & \makecell[r]{0.0086\\{\scriptsize$\pm$ 0.0035}} & 13.62 \\
Baseline & GP-EI & \makecell[r]{0.0201\\{\scriptsize$\pm$ 0.0043}} & \makecell[r]{0.0103\\{\scriptsize$\pm$ 0.0026}} & \makecell[r]{0.0140\\{\scriptsize$\pm$ 0.0048}} & \makecell[r]{0.0291\\{\scriptsize$\pm$ 0.0099}} & \makecell[r]{0.0104\\{\scriptsize$\pm$ 0.0019}} & \makecell[r]{0.0236\\{\scriptsize$\pm$ 0.0027}} & \makecell[r]{1.5944\\{\scriptsize$\pm$ 0.6797}} & \makecell[r]{0.0074\\{\scriptsize$\pm$ 0.0019}} & 14.62 \\
Baseline & GP-TS & \makecell[r]{0.0214\\{\scriptsize$\pm$ 0.0092}} & \makecell[r]{0.0105\\{\scriptsize$\pm$ 0.0026}} & \makecell[r]{0.0144\\{\scriptsize$\pm$ 0.0069}} & \makecell[r]{0.0287\\{\scriptsize$\pm$ 0.0079}} & \makecell[r]{0.0113\\{\scriptsize$\pm$ 0.0034}} & \makecell[r]{0.0361\\{\scriptsize$\pm$ 0.0208}} & \makecell[r]{1.6671\\{\scriptsize$\pm$ 0.9655}} & \makecell[r]{0.0070\\{\scriptsize$\pm$ 0.0034}} & 14.62 \\
Baseline & GP (BoTorch) & \makecell[r]{0.0244\\{\scriptsize$\pm$ 0.0090}} & \makecell[r]{0.0095\\{\scriptsize$\pm$ 0.0019}} & \makecell[r]{0.0210\\{\scriptsize$\pm$ 0.0069}} & \makecell[r]{0.0287\\{\scriptsize$\pm$ 0.0096}} & \makecell[r]{0.0108\\{\scriptsize$\pm$ 0.0039}} & \makecell[r]{0.0271\\{\scriptsize$\pm$ 0.0059}} & \makecell[r]{1.4880\\{\scriptsize$\pm$ 0.7358}} & \makecell[r]{0.0079\\{\scriptsize$\pm$ 0.0032}} & 15.25 \\
\bottomrule
\end{tabular}
\end{adjustbox}
\begin{tablenotes}[flushleft]
\footnotesize
\item Entries report mean $\pm$ std final simple regret over 10 seeds.
\item Bold cells indicate the best mean final regret within a benchmark column.
\end{tablenotes}
\end{threeparttable}
\end{table*}

\begin{table*}[t]
\centering
\caption{Synthetic BO appendix results for all method/acquisition variants in the extended sweep.
Methods are ordered by average rank.}
\label{tab:bo_all_methods_appendix}
\scriptsize
\setlength{\tabcolsep}{3.2pt}
\renewcommand{\arraystretch}{1.08}
\begin{threeparttable}
\begin{adjustbox}{width=0.9\textwidth}
\begin{tabular}{llcccccc}
\toprule
Family & Method & Branin & Hart-4D & Hart-6D & Ackley & \makecell{Noisy\\Ackley} & Avg. rank \\
\midrule
Dec-PFN & \textbf{Dec-PFN-LogEI} & \makecell[r]{0.0454\\{\scriptsize$\pm$ 0.0518}} & \textbf{\makecell[r]{0.0000\\{\scriptsize$\pm$ 0.0000}}} & \makecell[r]{0.2931\\{\scriptsize$\pm$ 0.4564}} & \textbf{\makecell[r]{0.0570\\{\scriptsize$\pm$ 0.0295}}} & \makecell[r]{0.0774\\{\scriptsize$\pm$ 0.0495}} & \textbf{3.80} \\
Dec-ICL & \textbf{Dec-ICL-LogEI} & \makecell[r]{0.0493\\{\scriptsize$\pm$ 0.0530}} & \textbf{\makecell[r]{0.0000\\{\scriptsize$\pm$ 0.0000}}} & \makecell[r]{0.4159\\{\scriptsize$\pm$ 0.1030}} & \makecell[r]{0.1211\\{\scriptsize$\pm$ 0.0758}} & \makecell[r]{0.1080\\{\scriptsize$\pm$ 0.0770}} & 4.40 \\
Tuned-PFN & Tuned-PFN-LogEI & \makecell[r]{0.0661\\{\scriptsize$\pm$ 0.0656}} & \textbf{\makecell[r]{0.0000\\{\scriptsize$\pm$ 0.0000}}} & \makecell[r]{0.2331\\{\scriptsize$\pm$ 0.3099}} & \makecell[r]{0.0607\\{\scriptsize$\pm$ 0.0521}} & \textbf{\makecell[r]{0.0373\\{\scriptsize$\pm$ 0.0325}}} & 4.40 \\
Baseline & GP (BoTorch) & \makecell[r]{0.0384\\{\scriptsize$\pm$ 0.0448}} & \textbf{\makecell[r]{0.0000\\{\scriptsize$\pm$ 0.0000}}} & \textbf{\makecell[r]{0.1718\\{\scriptsize$\pm$ 0.1010}}} & \multicolumn{1}{c}{--} & \makecell[r]{0.1381\\{\scriptsize$\pm$ 0.1947}} & 5.00 \\
Tuned-ICL & Tuned-ICL-LogEI & \textbf{\makecell[r]{0.0217\\{\scriptsize$\pm$ 0.0251}}} & \textbf{\makecell[r]{0.0000\\{\scriptsize$\pm$ 0.0000}}} & \makecell[r]{0.3840\\{\scriptsize$\pm$ 0.0876}} & \makecell[r]{0.1211\\{\scriptsize$\pm$ 0.0758}} & \makecell[r]{0.1080\\{\scriptsize$\pm$ 0.0770}} & 6.00 \\
Baseline & GP-EI (BoTorch) & \makecell[r]{0.0244\\{\scriptsize$\pm$ 0.0258}} & \textbf{\makecell[r]{0.0000\\{\scriptsize$\pm$ 0.0000}}} & \makecell[r]{0.3619\\{\scriptsize$\pm$ 0.6948}} & \makecell[r]{0.1691\\{\scriptsize$\pm$ 0.1003}} & \makecell[r]{0.3216\\{\scriptsize$\pm$ 0.2031}} & 9.20 \\
Baseline & GP-EI & \makecell[r]{0.0222\\{\scriptsize$\pm$ 0.0229}} & \textbf{\makecell[r]{0.0000\\{\scriptsize$\pm$ 0.0000}}} & \makecell[r]{0.1948\\{\scriptsize$\pm$ 0.2133}} & \makecell[r]{0.3709\\{\scriptsize$\pm$ 0.2515}} & \makecell[r]{0.5974\\{\scriptsize$\pm$ 0.3809}} & 9.60 \\
Dec-PFN & \textbf{Dec-PFN-TS} & \makecell[r]{0.0655\\{\scriptsize$\pm$ 0.0247}} & \textbf{\makecell[r]{0.0000\\{\scriptsize$\pm$ 0.0000}}} & \makecell[r]{0.4448\\{\scriptsize$\pm$ 0.1367}} & \makecell[r]{0.2489\\{\scriptsize$\pm$ 0.1354}} & \makecell[r]{0.2436\\{\scriptsize$\pm$ 0.1301}} & 10.20 \\
Dec-ICL & \textbf{Dec-ICL-EI} & \makecell[r]{0.0539\\{\scriptsize$\pm$ 0.0539}} & \textbf{\makecell[r]{0.0000\\{\scriptsize$\pm$ 0.0000}}} & \makecell[r]{0.4490\\{\scriptsize$\pm$ 0.1096}} & \makecell[r]{0.2640\\{\scriptsize$\pm$ 0.1494}} & \makecell[r]{0.2389\\{\scriptsize$\pm$ 0.1357}} & 10.60 \\
Dec-ICL & \textbf{Dec-ICL-LCB} & \makecell[r]{0.0425\\{\scriptsize$\pm$ 0.0482}} & \textbf{\makecell[r]{0.0000\\{\scriptsize$\pm$ 0.0000}}} & \makecell[r]{0.4336\\{\scriptsize$\pm$ 0.1088}} & \makecell[r]{0.2636\\{\scriptsize$\pm$ 0.1456}} & \makecell[r]{0.2377\\{\scriptsize$\pm$ 0.1331}} & 11.00 \\
Tuned-PFN & Tuned-PFN-TS & \makecell[r]{0.0368\\{\scriptsize$\pm$ 0.0255}} & \textbf{\makecell[r]{0.0000\\{\scriptsize$\pm$ 0.0000}}} & \makecell[r]{0.4070\\{\scriptsize$\pm$ 0.1513}} & \makecell[r]{0.2595\\{\scriptsize$\pm$ 0.1252}} & \makecell[r]{0.2486\\{\scriptsize$\pm$ 0.1317}} & 11.60 \\
Tuned-ICL & Tuned-ICL-TS & \makecell[r]{0.0935\\{\scriptsize$\pm$ 0.0524}} & \textbf{\makecell[r]{0.0000\\{\scriptsize$\pm$ 0.0000}}} & \makecell[r]{0.3513\\{\scriptsize$\pm$ 0.1727}} & \makecell[r]{0.2551\\{\scriptsize$\pm$ 0.1220}} & \makecell[r]{0.2460\\{\scriptsize$\pm$ 0.1300}} & 12.00 \\
Tuned-PFN & Tuned-PFN-LCB & \makecell[r]{0.0459\\{\scriptsize$\pm$ 0.0421}} & \textbf{\makecell[r]{0.0000\\{\scriptsize$\pm$ 0.0000}}} & \makecell[r]{0.4064\\{\scriptsize$\pm$ 0.1549}} & \makecell[r]{0.2456\\{\scriptsize$\pm$ 0.1335}} & \makecell[r]{0.2360\\{\scriptsize$\pm$ 0.1301}} & 12.20 \\
Dec-ICL & \textbf{Dec-ICL-TS} & \makecell[r]{0.0473\\{\scriptsize$\pm$ 0.0353}} & \textbf{\makecell[r]{0.0000\\{\scriptsize$\pm$ 0.0000}}} & \makecell[r]{0.4023\\{\scriptsize$\pm$ 0.1672}} & \makecell[r]{0.2929\\{\scriptsize$\pm$ 0.1265}} & \makecell[r]{0.2490\\{\scriptsize$\pm$ 0.1296}} & 12.60 \\
Dec-PFN & \textbf{Dec-PFN-LCB} & \makecell[r]{0.0548\\{\scriptsize$\pm$ 0.0392}} & \textbf{\makecell[r]{0.0000\\{\scriptsize$\pm$ 0.0000}}} & \makecell[r]{0.4111\\{\scriptsize$\pm$ 0.1599}} & \makecell[r]{0.2427\\{\scriptsize$\pm$ 0.1350}} & \makecell[r]{0.2489\\{\scriptsize$\pm$ 0.1354}} & 13.20 \\
Dec-ICL & \textbf{Dec-ICL-TS-Total} & \makecell[r]{0.0632\\{\scriptsize$\pm$ 0.0567}} & \textbf{\makecell[r]{0.0000\\{\scriptsize$\pm$ 0.0000}}} & \makecell[r]{0.4101\\{\scriptsize$\pm$ 0.1846}} & \makecell[r]{0.2705\\{\scriptsize$\pm$ 0.1383}} & \makecell[r]{0.2490\\{\scriptsize$\pm$ 0.1296}} & 13.60 \\
Dec-ICL & \textbf{Dec-ICL-LCB-Total} & \makecell[r]{0.0586\\{\scriptsize$\pm$ 0.0449}} & \textbf{\makecell[r]{0.0000\\{\scriptsize$\pm$ 0.0000}}} & \makecell[r]{0.4346\\{\scriptsize$\pm$ 0.1576}} & \makecell[r]{0.2533\\{\scriptsize$\pm$ 0.1464}} & \makecell[r]{0.2371\\{\scriptsize$\pm$ 0.1336}} & 13.80 \\
Tuned-ICL & Tuned-ICL-LCB & \makecell[r]{0.0569\\{\scriptsize$\pm$ 0.0304}} & \textbf{\makecell[r]{0.0000\\{\scriptsize$\pm$ 0.0000}}} & \makecell[r]{0.4507\\{\scriptsize$\pm$ 0.1137}} & \makecell[r]{0.2421\\{\scriptsize$\pm$ 0.1366}} & \makecell[r]{0.2367\\{\scriptsize$\pm$ 0.1298}} & 13.80 \\
Baseline & TPE (Optuna) & \makecell[r]{0.1103\\{\scriptsize$\pm$ 0.1000}} & \textbf{\makecell[r]{0.0000\\{\scriptsize$\pm$ 0.0000}}} & \makecell[r]{0.4510\\{\scriptsize$\pm$ 0.1883}} & \multicolumn{1}{c}{--} & \makecell[r]{0.8044\\{\scriptsize$\pm$ 0.6889}} & 14.25 \\
Baseline & GP-TS-Epi & \makecell[r]{0.0524\\{\scriptsize$\pm$ 0.0468}} & \textbf{\makecell[r]{0.0000\\{\scriptsize$\pm$ 0.0000}}} & \makecell[r]{0.6015\\{\scriptsize$\pm$ 0.3297}} & \multicolumn{1}{c}{--} & \multicolumn{1}{c}{--} & 16.67 \\
Baseline & Random & \makecell[r]{0.3535\\{\scriptsize$\pm$ 0.4968}} & \makecell[r]{0.0342\\{\scriptsize$\pm$ 0.0979}} & \makecell[r]{1.0019\\{\scriptsize$\pm$ 0.4135}} & \makecell[r]{2.0359\\{\scriptsize$\pm$ 0.8740}} & \makecell[r]{2.0359\\{\scriptsize$\pm$ 0.8740}} & 19.20 \\
Baseline & GP-TS & \makecell[r]{0.0680\\{\scriptsize$\pm$ 0.0452}} & \textbf{\makecell[r]{0.0000\\{\scriptsize$\pm$ 0.0000}}} & \makecell[r]{0.5966\\{\scriptsize$\pm$ 0.2847}} & \makecell[r]{0.6457\\{\scriptsize$\pm$ 0.3547}} & \makecell[r]{0.6380\\{\scriptsize$\pm$ 0.3471}} & 19.40 \\
\bottomrule
\end{tabular}
\end{adjustbox}
\begin{tablenotes}[flushleft]
\footnotesize
\item Entries report mean $\pm$ std final simple regret over 10 seeds (0--9).
\item Cells with fewer than 5 available seeds are shown as ``--'' and are excluded from the average-rank computation.
\item Bold cells indicate the best mean final regret within a benchmark column.
\end{tablenotes}
\end{threeparttable}
\end{table*}

\subsection{HPO: Per-Benchmark Regret Curves}
\label{app:hpo_curves}

\begin{figure}[ht]
    \centering
    \includegraphics[width=\linewidth]{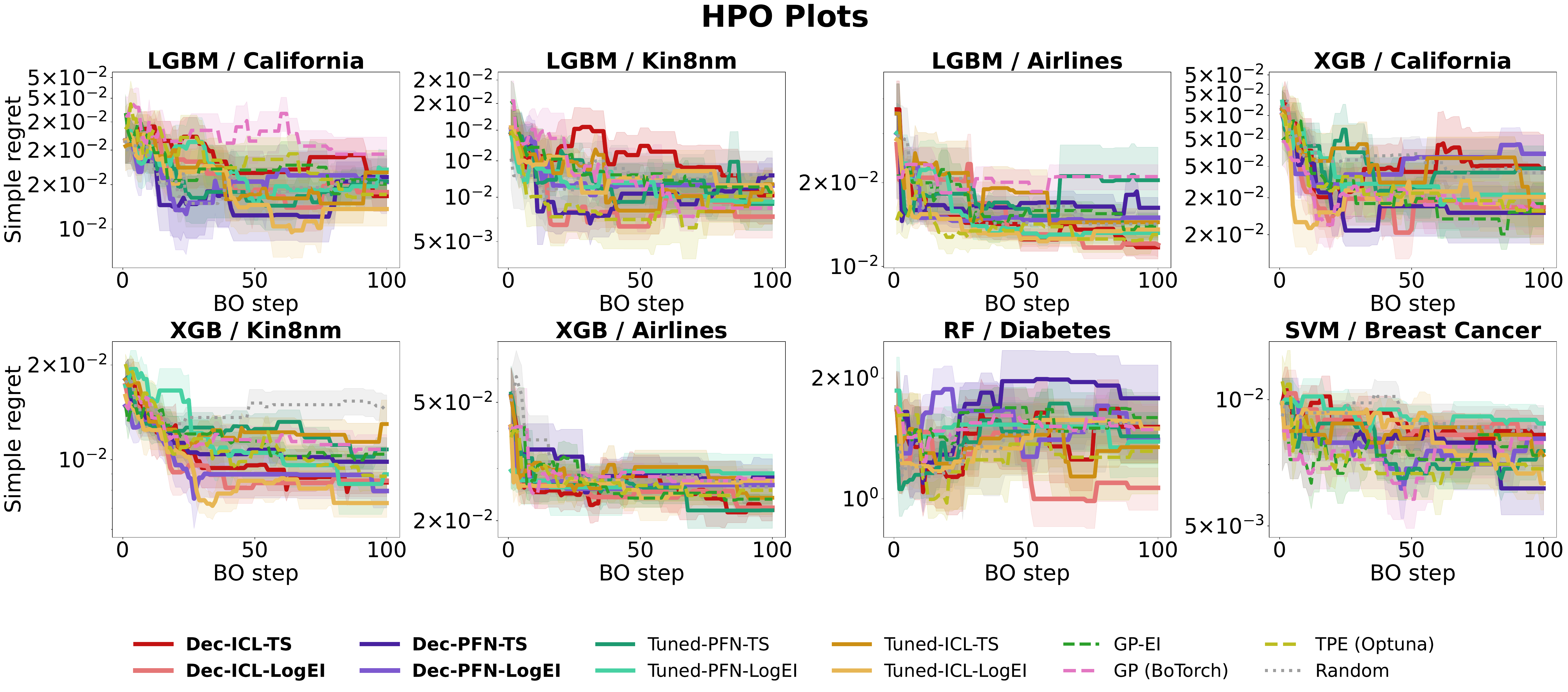}
    \caption{Simple regret over $100$ BO steps on all $8$ core HPO benchmarks,
    mean $\pm$ SE over $10$ seeds.
    Dec-ICL variants (red family, solid) consistently place among the top
    methods across LGBM and XGBoost tasks.
    The RF/Diabetes benchmark is visibly noisier due to its small dataset
    ($442$ examples) and stochastic train/val splits; most methods converge
    to similar final regret, making average rank a more reliable summary
    than individual curves.
    SVM/Breast Cancer shows the clearest separation, with
    Dec-ICL-LCB achieving the lowest final regret.}
    \label{fig:hpo_curves_appendix}
\end{figure}

\subsection{HPO: Full Method Comparison and LCB Ablation}
\label{app:hpo_full}

\begin{figure}[ht]
    \centering
    \includegraphics[width=\linewidth]{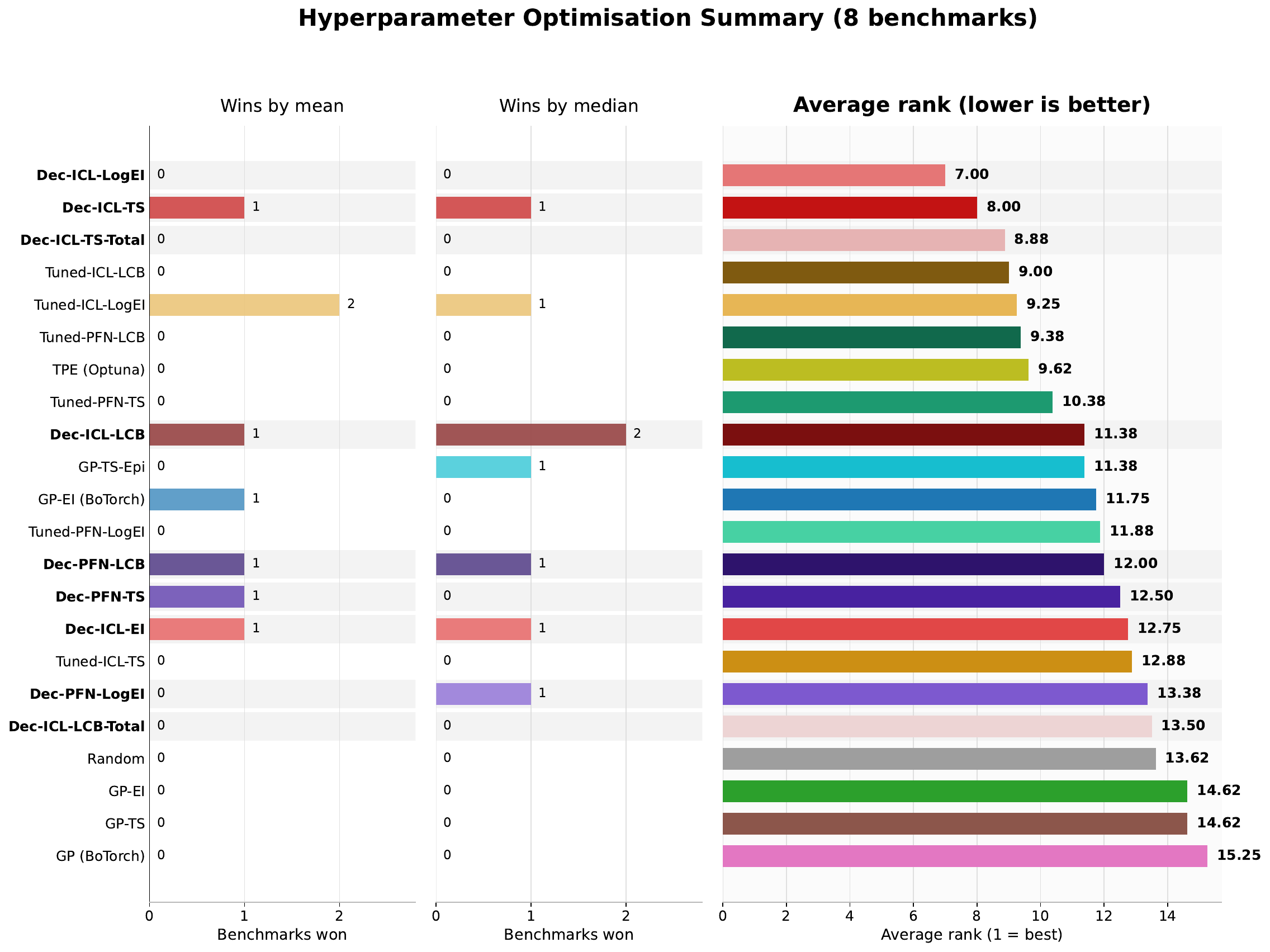}
    \caption{Full HPO summary including all acquisition variants.
    Methods are sorted by average rank (right panel).
    \textbf{Bold} labels indicate our proposed decoupled models.}
    \label{fig:hpo_full_appendix}
\end{figure}

Figure~\ref{fig:hpo_full_appendix} extends the main-text summary to include
all LCB and ablation variants.
Several observations are worth noting.

\paragraph{LCB ablation.}
Dec-ICL-LCB ranks $11.38$ on average across the $8$ HPO benchmarks,
notably weaker than Dec-ICL-LogEI ($7.00$) and Dec-ICL-TS ($8.00$)
despite using the same epistemic signal.
This is consistent with LCB being sensitive to its $\beta$ parameter in
high-dimensional or structured spaces where the confidence bound is harder
to calibrate.
The epistemic signal nonetheless matters: its total-variance counterpart
Dec-ICL-LCB-Total ranks $13.50$, more than two places worse, and
Dec-ICL-TS-Total ($8.88$) similarly trails Dec-ICL-TS ($8.00$).

\paragraph{Tuned vs.\ decoupled.}
For the ICL backbone, the decoupled models occupy the top two positions
(Dec-ICL-LogEI $7.00$, Dec-ICL-TS $8.00$) while the tuned counterparts
rank $9.25$ and $12.88$ respectively.

\paragraph{Baselines.}
TPE ($9.62$) is the strongest baseline on HPO by average rank, ahead of
all GP variants ($11.38$--$15.25$).
This is expected: TPE is specifically designed for HPO-style discrete/mixed
spaces and does not require a differentiable surrogate.

\subsection{Synthetic BO: Per-Benchmark Regret Curves}
\label{app:bo_curves}

\begin{figure}[ht]
    \centering
    \includegraphics[width=\linewidth]{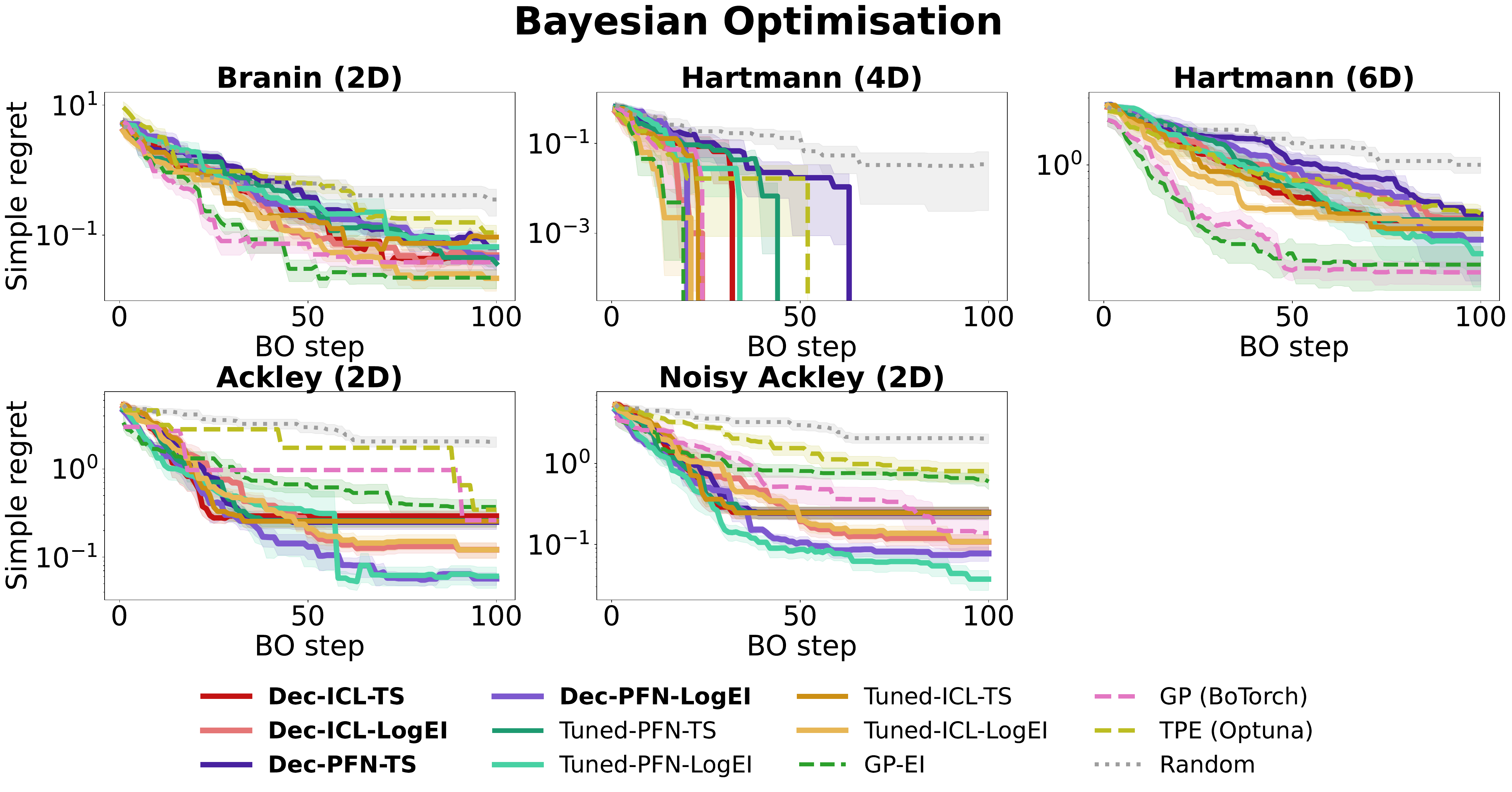}
    \caption{Simple regret over $100$ BO steps on the $5$ synthetic
    benchmarks, mean $\pm$ SE over $10$ seeds.
    LogEI variants (thicker solid lines) converge faster on all
    smooth benchmarks.}
    \label{fig:bo_curves_appendix}
\end{figure}

\subsection{Synthetic BO: Full Method Comparison}
\label{app:bo_full}

\begin{figure}[ht]
    \centering
    \includegraphics[width=\linewidth]{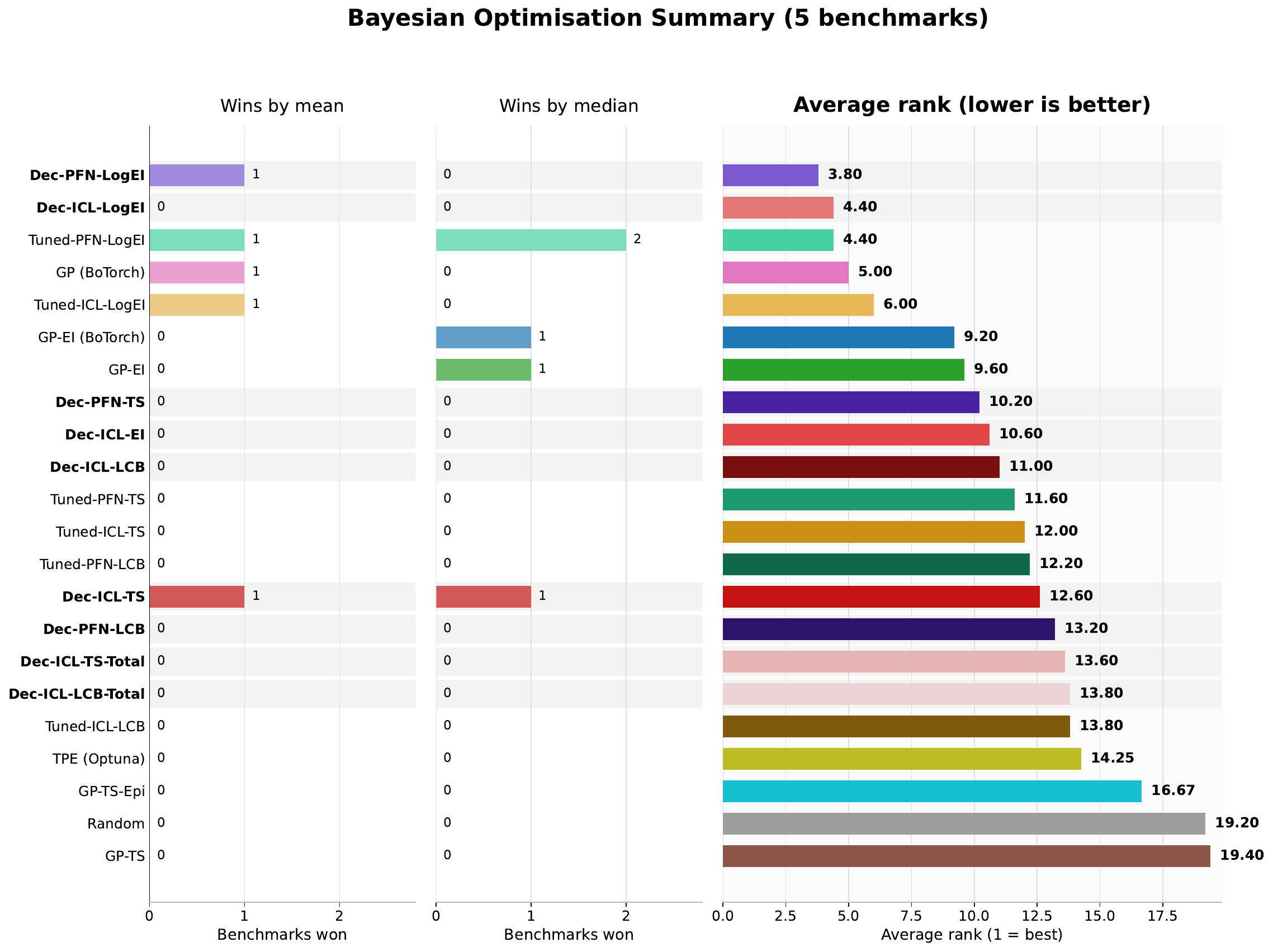}
    \caption{Full synthetic BO summary including all acquisition variants,
    sorted by average rank.
    \textbf{Bold} labels indicate our proposed decoupled models.}
    \label{fig:bo_full_appendix}
\end{figure}

Figure~\ref{fig:bo_full_appendix} shows the complete ranking over the
$5$ synthetic benchmarks.
The LogEI acquisition is the dominant factor: the top three methods
(Dec-PFN-LogEI $3.80$, Dec-ICL-LogEI $4.40$, Tuned-PFN-LogEI $4.40$)
all use gradient-based inner-loop optimisation, and the gap to the next
group (GP (BoTorch) $5.00$) is larger than any within-group difference.
By contrast, TS and LCB variants rank in the $10$--$13$ range regardless
of backbone, confirming that acquisition optimisation quality, not
model quality, is the bottleneck on smooth low-dimensional functions.
The total-variance ablations rank near the bottom of the field
(Dec-ICL-TS-Total $13.60$, Dec-ICL-LCB-Total $13.80$), well below their
epistemic counterparts, which isolates the contribution of the
decomposition even within the gradient-based regime.


\section{Additional Regression Results}
\label{app:regression}

Table~\ref{tab:regression_local_main} reports the observation-level regression
results for the decoupled and tuned models. Across datasets, the two variants
are broadly similar, with tuned models sometimes holding a small advantage in
NLL, which is expected because they are optimized solely for the noisy-target
predictive objective. These results suggest that the auxiliary decomposition
losses do not substantially harm predictive performance, but that the main
benefit of the proposed method lies in downstream decision-making rather than in
observation-only regression accuracy.

\begin{table*}[t]
\centering
\caption{
Local regression results for the four tuned models.
Each entry is mean $\pm$ std over 10 random train/test splits.
We use $n_{\mathrm{train}}=1024$ and $n_{\mathrm{test}}=3000$.
Lower is better.
Within each row, bold indicates the better method \emph{within the same backbone pair}:
Decoupled vs Tuned for TabICL, and Decoupled vs Tuned for TabPFN.
}
\label{tab:regression_local_main}
\small
\setlength{\tabcolsep}{5pt}
\renewcommand{\arraystretch}{1.12}
\begin{threeparttable}
\begin{adjustbox}{max width=\textwidth}
\begin{tabular}{llcccc}
\toprule
& & \multicolumn{2}{c}{TabICL Backbone} & \multicolumn{2}{c}{TabPFN Backbone} \\
\cmidrule(lr){3-4}\cmidrule(lr){5-6}
Dataset & Metric $\downarrow$ & \textbf{Decoupled (Ours)} & Tuned & \textbf{Decoupled (Ours)} & Tuned \\
\midrule

\multirow{2}{*}{California}
& RMSE
& 0.5049 $\pm$ 0.0160
& 0.5047 $\pm$ 0.0155
& 0.5082 $\pm$ 0.0139
& 0.5077 $\pm$ 0.0150 \\
& NLL
& 0.488 $\pm$ 0.042
& 0.480 $\pm$ 0.040
& 0.515 $\pm$ 0.037
& 0.505 $\pm$ 0.040 \\
\midrule

\multirow{2}{*}{Kin8nm}
& RMSE
& 0.0879 $\pm$ 0.0024
& 0.0878 $\pm$ 0.0025
& 0.1014 $\pm$ 0.0037
& 0.0989 $\pm$ 0.0038 \\
& NLL
& -1.050 $\pm$ 0.027
& -1.051 $\pm$ 0.027
& -0.938 $\pm$ 0.036
& -0.961 $\pm$ 0.034 \\
\midrule

\multirow{2}{*}{Airlines}
& RMSE
& 1.9512 $\pm$ 0.0170
& 1.9521 $\pm$ 0.0161
& 1.9622 $\pm$ 0.0211
& 1.9469 $\pm$ 0.0170 \\
& NLL
& 2.085 $\pm$ 0.011
& 2.082 $\pm$ 0.009
& 2.093 $\pm$ 0.012
& 2.087 $\pm$ 0.015 \\
\midrule

\multirow{2}{*}{Year}
& RMSE
& 9.3594 $\pm$ 0.1155
& 9.4008 $\pm$ 0.1302
& 9.4035 $\pm$ 0.1337
& 9.3967 $\pm$ 0.0963 \\
& NLL
& 3.537 $\pm$ 0.016
& 3.537 $\pm$ 0.017
& 3.547 $\pm$ 0.028
& 3.546 $\pm$ 0.023 \\
\midrule

\multirow{2}{*}{Wine}
& RMSE
& 0.6967 $\pm$ 0.0078
& 0.6982 $\pm$ 0.0071
& 0.7457 $\pm$ 0.0156
& 0.7440 $\pm$ 0.0158 \\
& NLL
& 1.288 $\pm$ 0.225
& 1.263 $\pm$ 0.222
& 1.207 $\pm$ 0.081
& 1.295 $\pm$ 0.099 \\
\bottomrule
\end{tabular}
\end{adjustbox}
\begin{tablenotes}[flushleft]
\footnotesize
\item NLL denotes Gaussian negative log-likelihood computed from the predicted mean and total predictive variance.

\end{tablenotes}
\end{threeparttable}
\end{table*}

\section{Additional Active Learning Details}
\label{app:active_learning}

\subsection{Protocol}
\label{app:al_protocol}

We evaluate pool-based active learning on five real regression datasets:
California Housing, Kin8nm, NYC Taxi, Airlines, and YearPredictionMSD.
For each dataset, we construct a fixed unlabeled pool and a fixed held-out test
set, and then run \emph{sequential} active learning. Specifically, we begin from
an initial labelled context set of size \(n_0 = 1024\), and then acquire
\(n_q = 1024\) additional points \emph{one at a time}. Thus, the final reported
performance corresponds to 1024 sequential acquisitions rather than a single
1024-point batch selection. For all real-data runs, the unlabeled pool size is
5000 and the test set size is 3000.

For a given dataset and seed, the pool/test split is fixed across acquisition
methods, and the same warm-start indices are reused across methods. This ensures
that differences between acquisition rules are not confounded by different data
splits or different initial labelled sets. The active-learning runner stores the
pool/test split, per-step trajectories, and per-run summaries for reproducibility.

Table~\ref{tab:al_protocol_real} summarises the exact protocol used for the
real-data experiments reported in the paper.

\begin{table}[t]
\centering
\caption{Protocol for the real-data active-learning experiments.}
\label{tab:al_protocol_real}
\small
\setlength{\tabcolsep}{5pt}
\renewcommand{\arraystretch}{1.08}
\begin{tabular}{lcccccc}
\toprule
Dataset & Dim & \(n_0\) & \(n_q\) & Pool & Test & Seeds \\
\midrule
California Housing & 8  & 1024 & 1024 & 5000 & 3000 & 16 \\
Kin8nm             & 8  & 1024 & 1024 & 5000 & 3000 & 20 \\
NYC Taxi           & 7  & 1024 & 1024 & 5000 & 3000 & 19 \\
Airlines           & 5  & 1024 & 1024 & 5000 & 3000 & 20 \\
YearPredictionMSD  & 90 & 1024 & 1024 & 5000 & 3000 & 20 \\
\bottomrule
\end{tabular}
\end{table}

\paragraph{Evaluation metrics.}
At each acquisition step we record test RMSE, MAE, Gaussian NLL, CRPS, mean
epistemic / aleatoric / total variance, and empirical coverage and interval
widths at 50\%, 80\%, 90\%, and 95\%. The summary tables in the main text report
the final-step metrics after all \(n_q\) sequential acquisitions have been made.

\subsection{Acquisition Scores}
\label{app:al_acq_defs}

The main-text matched comparison isolates the effect of the uncertainty signal:
decoupled models acquire using the latent epistemic variance, whereas tuned
baselines acquire using total predictive variance. In the broader appendix
ablation below, we additionally compare several acquisition heuristics on the
decoupled TabICL backbone.

For a candidate point \(x\), let \(v_{\mathrm{epi}}(x)\) denote the latent
epistemic variance, \(v_{\mathrm{tot}}(x)\) the total predictive variance, and
\(p^{(y)}(x)\) the induced observation-level binned predictive distribution.

\paragraph{Epistemic uncertainty sampling.}
The epistemic acquisition score is
\[
\alpha_{\mathrm{epi}}(x) = v_{\mathrm{epi}}(x),
\]
so points with the largest latent epistemic variance are selected.

\paragraph{Total uncertainty sampling.}
The total-variance acquisition score is
\[
\alpha_{\mathrm{tot}}(x) = v_{\mathrm{tot}}(x),
\]
which corresponds to standard predictive-variance sampling.

\paragraph{BALD.}
For the decoupled binned model, BALD is computed from the induced predictive
distribution and the latent-noise transition model. Writing
\[
p_k^{(y)}(x)=\sum_{j=1}^K \pi_j^{(f)}(x)\,T_{k\mid j}(x),
\]
the score is
\[
\alpha_{\mathrm{BALD}}(x)
=
H\!\bigl(p^{(y)}(x)\bigr)
-
\sum_{j=1}^K \pi_j^{(f)}(x)\,H\!\bigl(T_{\cdot\mid j}(x)\bigr),
\]
where \(H(\cdot)\) denotes entropy. This is the mutual-information-style score
implemented in the active-learning runner.

\paragraph{EPIG proxy.}
For EPIG we use a representation-based proxy. Let \(r(x)\) denote the normalized
final representation of candidate \(x\), and define the epistemic fraction
\[
\eta(x)=\frac{v_{\mathrm{epi}}(x)}{v_{\mathrm{tot}}(x)}.
\]
Given a target subset \(\mathcal X_{\mathrm{tar}}\) drawn from the test
distribution, we define
\[
\rho^2(x,x')
=
\bigl(r(x)^\top r(x')\bigr)^2\,\eta(x)\,\eta(x'),
\]
and score each candidate by
\[
\alpha_{\mathrm{EPIG}}(x)
=
\frac{1}{|\mathcal X_{\mathrm{tar}}|}
\sum_{x' \in \mathcal X_{\mathrm{tar}}}
\left(
-\frac{1}{2}\log(1-\rho^2(x,x'))
\right).
\]
In the real-data runs reported here, the target subset size is 2028 points,
sampled from the 3000-point test set for each seed.

\paragraph{Random.}
The random baseline samples uniformly from the remaining pool.

\subsection{Additional Real-Data Acquisition Comparison}
\label{app:al_real_ablation}

Table~\ref{tab:al_real_ablation_rmse} complements the main-text matched
comparison by evaluating several acquisition heuristics on the same decoupled
TabICL model. This ablation addresses a different question from the main text:
rather than comparing decoupled and tuned uncertainty signals within the same
matched setup, it compares several acquisition objectives on a fixed decoupled
model. As a result, the conclusions are naturally more heterogeneous.

\begin{table*}[t]
\centering
\caption{
Additional active-learning acquisition comparison on the real datasets using the
decoupled TabICL backbone. Values are mean $\pm$ std final RMSE over the seeds
listed in Table~\ref{tab:al_protocol_real}. Lower is better. Boldface marks the
best acquisition score within each dataset.
}
\label{tab:al_real_ablation_rmse}
\small
\setlength{\tabcolsep}{5pt}
\renewcommand{\arraystretch}{1.08}
\begin{adjustbox}{max width=\textwidth}
\begin{tabular}{lccccc}
\toprule
Dataset & Random & Tuned-ICL &\textbf{ BALD (ours)} & \textbf{EPIG (ours)} & \textbf{Epistemic (ours)} \\
\midrule
California Housing
& 0.4793 $\pm$ 0.0059
& 0.4783 $\pm$ 0.0033
& 0.4787 $\pm$ 0.0044
& \textbf{0.4752 $\pm$ 0.0058}
& 0.4772 $\pm$ 0.0035 \\
Kin8nm
& 0.0787 $\pm$ 0.0013
& 0.0805 $\pm$ 0.0012
& 0.0790 $\pm$ 0.0011
& \textbf{0.0784 $\pm$ 0.0011}
& 0.0796 $\pm$ 0.0010 \\
NYC Taxi
& 3.3588 $\pm$ 0.1177
& 3.1296 $\pm$ 0.0157
& \textbf{3.1208 $\pm$ 0.0257}
& 3.2786 $\pm$ 0.0868
& 3.1275 $\pm$ 0.0140 \\
Airlines
& 1.9124 $\pm$ 0.0059
& \textbf{1.9040 $\pm$ 0.0037}
& 1.9042 $\pm$ 0.0034
& 1.9064 $\pm$ 0.0044
& 1.9046 $\pm$ 0.0028 \\
YearPredictionMSD
& 9.3926 $\pm$ 0.0693
& 9.4125 $\pm$ 0.0306
& \textbf{9.3693 $\pm$ 0.0441}
& 9.4272 $\pm$ 0.0493
& 9.4150 $\pm$ 0.0310 \\
\bottomrule
\end{tabular}
\end{adjustbox}
\end{table*}

\paragraph{Interpretation.}
The broader acquisition ablation is mixed, which is typical in active learning:
no single heuristic dominates on every dataset. Nevertheless, epistemic
acquisition remains consistently competitive and improves substantially over the
random baseline on California, NYC Taxi, and Airlines. The matched comparison in
the main text should therefore be interpreted as isolating the value of the
decoupled uncertainty signal relative to a tuned total-uncertainty baseline,
while Table~\ref{tab:al_real_ablation_rmse} provides a wider acquisition-level
context.

\subsection{Learning Curves}
\label{app:al_curves}

The final-step comparisons in the main text do not show when an acquisition rule
begins to separate from the alternatives. We therefore also inspect the full
RMSE trajectories over the 1024 sequential acquisitions.

\begin{figure*}[t]
    \centering
    \includegraphics[width=\linewidth]{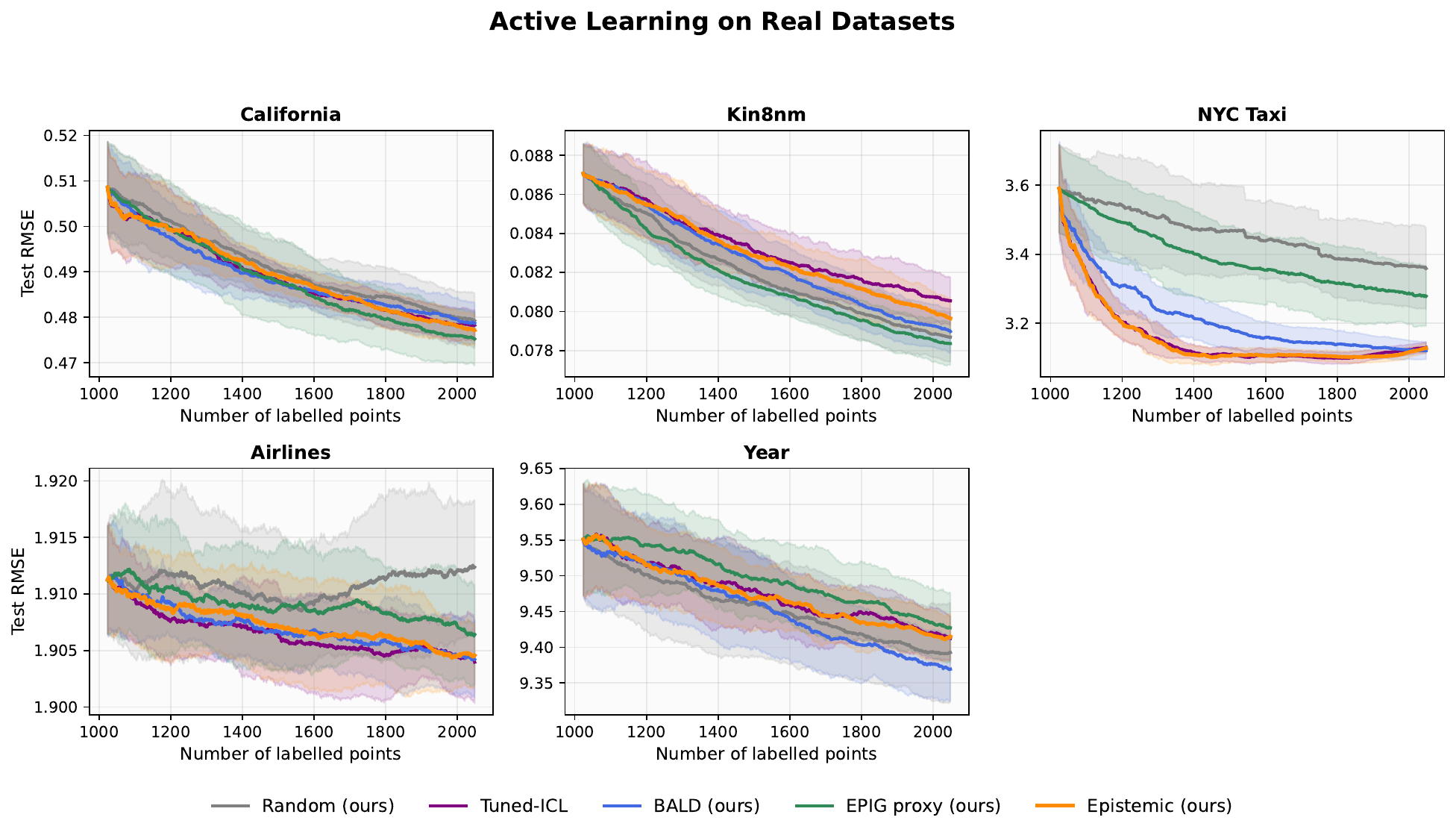}
    \caption{
    Test RMSE as a function of the number of labelled points for the real-data
    active-learning experiments. Curves show mean $\pm$ std over seeds. The
    trajectories complement the final-step tables by showing whether the gains
    from a given acquisition rule appear early, late, or only after many
    sequential acquisitions.
    }
    \label{fig:al_real_rmse_curves}
\end{figure*}

\subsection{Reproducibility Details}
\label{app:al_repro}

For each dataset and seed, the active-learning pipeline stores:
(i) the fixed pool/test split,
(ii) the warm-start indices,
(iii) per-step acquisition trajectories in JSONL format,
(iv) per-run summary files containing the final metrics, and
(v) curve arrays containing RMSE, NLL, MAE, and calibration summaries across
acquisition steps. These saved artifacts make it possible to regenerate the
appendix plots and to recompute alternative summaries without rerunning the full
active-learning experiments.

\section{Limitations, Broader Impact, and Reproducibility Notes}
\label{app:limitations_impact_repro}

\paragraph{Limitations.}
Our method relies on a synthetic task prior that explicitly defines a latent
signal and an observation-noise process. Its effectiveness therefore depends on
how well this synthetic decomposition transfers to real downstream tasks. The
experiments cover several BO, HPO, and active-learning settings, but they do not
exhaustively characterize sensitivity to the synthetic prior, auxiliary loss
weights, bin discretization, or non-Gaussian noise models. The current
implementation also uses moment-matched Gaussian acquisition functions for
binned PFN outputs, which is a practical approximation rather than an exact
acquisition rule for the full binned predictive distribution.

\paragraph{Broader impact.}
This work is foundational research on uncertainty estimation and sequential
decision-making. Potential positive impacts include more sample-efficient
optimization and active learning in settings where evaluations are expensive or
noisy. Potential negative impacts arise if uncertainty estimates are
over-trusted in high-stakes applications, or if more efficient optimization is
used to accelerate undesirable model development. We do not deploy the method in
sensitive decision-making domains, and our empirical evaluation is limited to
standard tabular, synthetic BO, and HPO benchmarks.

\paragraph{Compute resources.}
Experiments were run on GPU compute nodes using public pretrained TabPFN and
TabICL checkpoints, with additional CPU-based evaluations for classical
baselines and HPO objective evaluations. The main cost comes from fine-tuning the
four PFN/ICL variants and running the sequential BO, HPO, and active-learning
loops over multiple seeds. Exact hardware details, wall-clock times, and total
GPU-hour estimates should be reported in the released code or supplementary
material.

\paragraph{Existing assets.}
The work uses public pretrained TabPFN and TabICL checkpoints, public benchmark
datasets, and standard open-source packages including BoTorch and Optuna. These
assets are cited in the paper. License and terms-of-use information should be
checked against the corresponding public repositories or dataset sources before
final submission.

\newpage
\section*{NeurIPS Paper Checklist}

\begin{enumerate}

\item {\bf Claims}
    \item[] Question: Do the main claims made in the abstract and introduction accurately reflect the paper's contributions and scope?
    \item[] Answer: \answerYes{}.
    \item[] Justification: The abstract and introduction state the main theoretical claim, namely non-identifiability of the epistemic--aleatoric split from the marginal predictive alone, and the main methodological contribution, namely a decoupled PFN trained with privileged latent-signal and noise supervision. The experimental claims are stated with appropriate scope, emphasizing matched comparisons, broader sweeps, and the settings where gains are clearest.


\item {\bf Limitations}
    \item[] Question: Does the paper discuss the limitations of the work performed by the authors?
    \item[] Answer: \answerYes{}.
    \item[] Justification: The conclusion and Appendix~\ref{app:limitations_impact_repro} discuss limitations, including dependence on the synthetic task prior, sensitivity to auxiliary losses and binning, the use of Gaussian moment-matched acquisitions, and the limited coverage of possible real-world noise processes.

\item {\bf Theory assumptions and proofs}
    \item[] Question: For each theoretical result, does the paper provide the full set of assumptions and a complete (and correct) proof?
    \item[] Answer: \answerYes{}.
    \item[] Justification: Proposition~\ref{prop:non_identifiability} states the assumptions for the Gaussian additive non-identifiability result and provides a proof in the main text. Proposition~\ref{prop:population_consistency} and Corollary~\ref{cor:epistemic_consistency} state the assumptions for the population-risk interpretation of the auxiliary heads and provide proofs in Appendix~\ref{app:consistency}.

    \item {\bf Experimental result reproducibility}
    \item[] Question: Does the paper fully disclose all the information needed to reproduce the main experimental results of the paper to the extent that it affects the main claims and/or conclusions of the paper (regardless of whether the code and data are provided or not)?
    \item[] Answer: \answerYes{}.
    \item[] Justification: The paper describes the synthetic task prior, model fine-tuning protocol, loss functions, acquisition functions, benchmark definitions, search spaces, active-learning protocol, seeds, and evaluation metrics in Sections~\ref{sec:method}--\ref{sec:active_learning} and Appendices~\ref{app:synthetic_pretraining}--\ref{app:active_learning}. These details specify the main ingredients needed to reproduce the reported comparisons.

\item {\bf Open access to data and code}
    \item[] Question: Does the paper provide open access to the data and code, with sufficient instructions to faithfully reproduce the main experimental results, as described in supplemental material?
    \item[] Answer: \answerNo{}.
    \item[] Justification: The paper provides detailed methodological and experimental descriptions, but the current submission does not include a public anonymized code release with full reproduction scripts. We plan to release code and instructions after review, in the camera-ready.

\item {\bf Experimental setting/details}
    \item[] Question: Does the paper specify all the training and test details (e.g., data splits, hyperparameters, how they were chosen, type of optimizer) necessary to understand the results?
    \item[] Answer: \answerYes{}.
    \item[] Justification: The main text and appendices specify the fine-tuning setup, optimizer, learning rates, batch sizes, loss weights, binning scheme, partial-unfreezing protocol, benchmark search spaces, dataset sizes, seed counts, acquisition rules, and evaluation metrics. Further details are provided in Appendices~\ref{app:synthetic_pretraining}, \ref{app:hpo_benchmarks}, \ref{app:acq_binned_pfn}, and \ref{app:active_learning}.

\item {\bf Experiment statistical significance}
    \item[] Question: Does the paper report error bars suitably and correctly defined or other appropriate information about the statistical significance of the experiments?
    \item[] Answer: \answerYes{}.
    \item[] Justification: The main and appendix tables report mean $\pm$ standard deviation over seeds or data splits, and the learning-curve figures report mean trajectories with uncertainty bands. The captions specify whether variability is computed over seeds, random train/test splits, or dataset-specific active-learning runs.

\item {\bf Experiments compute resources}
    \item[] Question: For each experiment, does the paper provide sufficient information on the computer resources (type of compute workers, memory, time of execution) needed to reproduce the experiments?
    \item[] Answer: \answerNo{}.
    \item[] Justification: The paper describes the experimental protocol and optimization details, but it does not yet provide exact hardware types, memory, wall-clock time, or total GPU-hour estimates for each experiment. Appendix~\ref{app:limitations_impact_repro} gives a high-level compute description, and exact compute information will be included in the released reproduction package if available.
    
\item {\bf Code of ethics}
    \item[] Question: Does the research conducted in the paper conform, in every respect, with the NeurIPS Code of Ethics \url{https://neurips.cc/public/EthicsGuidelines}?
    \item[] Answer: \answerYes{}.
    \item[] Justification: The work uses standard public datasets, public pretrained model checkpoints, and synthetic benchmarks, and does not involve human-subject experiments, private data, or deployment in sensitive domains. The paper also discusses limitations and possible broader impacts in Appendix~\ref{app:limitations_impact_repro}.

\item {\bf Broader impacts}
    \item[] Question: Does the paper discuss both potential positive societal impacts and negative societal impacts of the work performed?
    \item[] Answer: \answerYes{}.
    \item[] Justification: Appendix~\ref{app:limitations_impact_repro} discusses potential positive impacts, such as more sample-efficient optimization and active learning, as well as potential negative impacts from over-trusting uncertainty estimates or accelerating undesirable model development. The work is foundational and is not deployed in high-stakes applications.
    
\item {\bf Safeguards}
    \item[] Question: Does the paper describe safeguards that have been put in place for responsible release of data or models that have a high risk for misuse (e.g., pre-trained language models, image generators, or scraped datasets)?
    \item[] Answer: \answerNA{}.
    \item[] Justification: The paper does not release high-risk models such as large language models, image generators, or scraped datasets. The method is evaluated on synthetic, tabular, BO, HPO, and active-learning benchmarks, and does not introduce a high-risk dual-use artifact.

\item {\bf Licenses for existing assets}
    \item[] Question: Are the creators or original owners of assets (e.g., code, data, models), used in the paper, properly credited and are the license and terms of use explicitly mentioned and properly respected?
    \item[] Answer: \answerNo{}.
    \item[] Justification: The paper cites the public pretrained TabPFN and TabICL checkpoints, public benchmark datasets, and software packages such as BoTorch and Optuna. However, the current draft does not explicitly list all corresponding licenses and terms of use, so these should be checked and added before final submission.

\item {\bf New assets}
    \item[] Question: Are new assets introduced in the paper well documented and is the documentation provided alongside the assets?
    \item[] Answer: \answerNA{}.
    \item[] Justification: The current submission introduces a method and trained variants used for experiments, but it does not release a new dataset, benchmark suite, or public model asset as a primary contribution. If code or checkpoints are released, they should be accompanied by documentation, usage instructions, and license information.

\item {\bf Crowdsourcing and research with human subjects}
    \item[] Question: For crowdsourcing experiments and research with human subjects, does the paper include the full text of instructions given to participants and screenshots, if applicable, as well as details about compensation (if any)? 
    \item[] Answer: \answerNA{}.
    \item[] Justification: The paper does not involve crowdsourcing, user studies, human-subject experiments, or participant compensation.

\item {\bf Institutional review board (IRB) approvals or equivalent for research with human subjects}
    \item[] Question: Does the paper describe potential risks incurred by study participants, whether such risks were disclosed to the subjects, and whether Institutional Review Board (IRB) approvals (or an equivalent approval/review based on the requirements of your country or institution) were obtained?
    \item[] Answer: \answerNA{}.
    \item[] Justification: The paper does not involve crowdsourcing or human-subject research, so IRB approval or equivalent review is not applicable.

\item {\bf Declaration of LLM usage}
    \item[] Question: Does the paper describe the usage of LLMs if it is an important, original, or non-standard component of the core methods in this research? Note that if the LLM is used only for writing, editing, or formatting purposes and does \emph{not} impact the core methodology, scientific rigor, or originality of the research, declaration is not required.
    \item[] Answer: \answerNA{}.
    \item[] Justification: The core method, experiments, and scientific claims do not rely on LLMs as an important, original, or non-standard component. Any use of language tools for writing, editing, or formatting does not affect the methodology, experimental results, or originality of the research.

\end{enumerate}

\end{document}